\RequirePackage{fix-cm}

\documentclass[smallextended]{svjour3}       
\smartqed  

\let\vec\relax 
\DeclareMathAccent{\vec}{\mathord}{letters}{"7E} 

\usepackage{graphicx}
\usepackage{natbib}
\usepackage{amsmath,amssymb,mathtools}
\usepackage{bbm}
\usepackage{enumerate}
\usepackage{booktabs}
\usepackage{amsthm}
\usepackage[export]{adjustbox}
\usepackage{algorithm}
\usepackage[noend]{algpseudocode}
\usepackage{thmtools, thm-restate}
\usepackage{afterpage}
\usepackage{pdflscape}
\usepackage{adjustbox}

\usepackage{mathtools}
\usepackage{multirow}
\usepackage{wrapfig}
\usepackage{comment}
\usepackage{xcolor}
\usepackage{footmisc}
\usepackage{tcolorbox}
\usepackage[caption=false]{subfig}
\usepackage[normalem]{ulem}
\usepackage{siunitx,booktabs}
\usepackage{tabularx,ragged2e}
\newcolumntype{C}{>{\Centering\arraybackslash}X} 
\usepackage{colortbl}

\theoremstyle{definition}
\newtheorem{defn}{Definition}

\DeclareMathOperator*{\argmax}{\arg\max}

\newcommand{\sety}{\hat{Y}}
\newcommand{\bx}{\boldsymbol{x}}
\newcommand{\bw}{\boldsymbol{w}}
\newcommand{\calY}{\mathcal{Y}}
\newcommand{\calQ}{\mathcal{Q}}

\definecolor{dgreen}{rgb}{0.0, 0.5, 0.0}
\newcommand{\ubopmnisttable}{
\adjustbox{width=0.55\columnwidth,valign=B,raise=1\baselineskip}{%
    \renewcommand{\arraystretch}{1.5}%
    \begin{tabular}{l|ccc}
         & \includegraphics[width=10mm,valign=c]{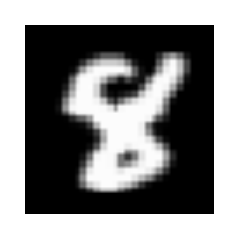} & \includegraphics[width=10mm,valign=c]{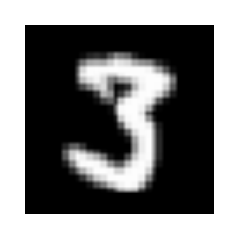} & \includegraphics[width=10mm,valign=c]{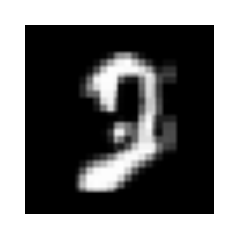} \\
         \hline
         $\beta = 5$ & $\{\underline{8}\}$ & $\{\underline{3},5,9\}$ & $\{9,3,\underline{2},7,8,1\}$  \\
         $\beta = 2$ & $\{\underline{8}\}$ & $\{\underline{3}\}$ & $\{9,3,\underline{2}\}$  \\
         $\beta = 1$ & $\{\underline{8}\}$ & $\{\underline{3}\}$ & $\{9,3\}$  \\
         \hline
         & $H = 10^{-5}$ & $H = 0.56$ & $H = 2.54$ \\
    \end{tabular}
    }
}
\newcommand{\ubopdbpediatable}{
\adjustbox{width=\columnwidth,valign=B,raise=1\baselineskip}{%
    \renewcommand{\arraystretch}{1.5}%
  \centering
  \setlength\tabcolsep{1pt}
	\begin{tabularx}{\textwidth}{CCC}
	\hline
	\multicolumn{1}{|C|}{\cellcolor{gray!15}\texttt{Lance Ten Broeck (born March 21, 1956) is an American professional golfer who has played on the PGA Tour, Nationwide Tour, and Champions Tour. Ten Broeck was born in Chicago, Illinois, and grew-up in Beverly, a community on the city's southwest side...}}  & \multicolumn{1}{C|}{\cellcolor{gray!15}\texttt{MTS TV is a digital television service owned by Telekom Srbija. The service provides thematic channels, HD channels, video on demand, video recording, the use of an Electronic Program Guide (EPG) and other services...}} & \multicolumn{1}{C|}{\cellcolor{gray!15}\texttt{Liao Bingxiong was a Chinese political cartoonist, painter and calligrapher. He remained active from 1934 until he gave up in 1995 (with a 20-year break between 1957 and 1978). Liao is widely regarded as one of China's foremost political cartoonists...}}  \\	
	\hline
	\scriptsize\it\{\textbf{GolfTournament, \underline{GolfPlayer}}, RugbyPlayer, SoccerPlayer, CyclingRace\} & \scriptsize\it\{\textbf{TelevisionStation, \underline{BroadcastNetwork}}, BusCompany, Comedian, RailwayLine\} & \scriptsize\it\{\textbf{Painter, \underline{ComicsCreator}}, PoliticalParty, President, Philosopher\}
	\end{tabularx}
    }
}
\renewcommand{\vec}[1]{\mathbf{#1}} 

\begin{document}

\title{Efficient Set-Valued Prediction in Multi-Class Classification
}


\author{Thomas Mortier          \and
        Marek Wydmuch           \and
        Krzysztof Dembczy\'nski \and
        Eyke H\"ullermeier      \and
        Willem Waegeman         
}


\institute{
            T. Mortier and W. Waegeman \at
              Coupure links 653, 9000 Ghent, Belgium \\
              Tel.: + 32 9 264 59 32\\
              \email{\{thomasf.mortier,willem.waegeman\}@ugent.be}           
           \and
            M. Wydmuch and K. Dembczy\'nski \at
              Piotrowo 2, 60-965 Pozna\'n, Poland \\
              \email{\{mwydmuch,kdembczynski\}@cs.put.poznan.pl}           
           \and
            E. H\"ullermeier \at
              Pohlweg 51, 33098 Paderborn, Germany \\
              \email{eyke@upb.de}           
}

\date{Received: date / Accepted: date}

\maketitle

\begin{abstract}
In cases of uncertainty, a multi-class classifier preferably returns a set of candidate classes instead of predicting a single class label with little guarantee. More precisely, the classifier should strive for an optimal balance between the correctness (the true class is among the candidates) and the precision (the candidates are not too many) of its prediction. We formalize this problem within a general decision-theoretic framework that unifies most of the existing work in this area. In this framework, uncertainty is quantified in terms of conditional class probabilities, and the quality of a predicted set is measured in terms of a utility function. We then address the problem of finding the Bayes-optimal prediction, i.e., the subset of class labels with highest expected utility. For this problem, which is computationally challenging as there are exponentially (in the number of classes) many predictions to choose from, we propose efficient algorithms that can be applied to a broad family of utility functions. Our theoretical results are complemented by experimental studies, in which we analyze the proposed algorithms in terms of predictive accuracy and runtime efficiency. 

\keywords{Set-valued prediction \and Multi-class classification \and Expected utility maximization}
\end{abstract}

\section{Introduction}
\label{sec:intro}
In probabilistic multi-class classification, one often encounters situations in which the classifier is uncertain about the class label for a given instance. 
In such cases, instead of predicting a single class, it might be beneficial to return a set of classes as a prediction, with the idea that the correct class should at least be contained in that set.  For example, in medical diagnosis, when not being sure enough about the true disease of a patient, it is better to return a set of candidate diseases. Provided this set is sufficiently small compared to the total number of diagnoses, it can still be of great help for a medical doctor, because only the remaining candidate diseases need further investigation. 

Let us introduce set-valued prediction in a formal way. We assume training examples $\{(\bx_{i},y_{i})\}_{i=1}^N$ from a distribution $P(\bx,y)$ on $\mathcal{X}\times\mathcal{Y}$, with $
\mathcal{X}$ some instance space (e.g., images, documents, etc.) and $\mathcal{Y}=\{c_1,\ldots,c_K\}$ a class space consisting of $K$
classes. In a probabilistic multi-class classification setting, we estimate the conditional class probabilities $P( \cdot \,|\,\bx)$ over $\mathcal{Y}$, with properties $
\forall c \in \mathcal{Y}: 0 \leq P(c\,|\,\bx) \leq 1 \,, \sum_{c \in \mathcal{Y}} P(c\,|\, \bx) = 1 \,.$ This distribution can be estimated using a wide range of well-known probabilistic methods (see further below). We will consider a set-valued prediction $\sety$ from the power set of $\mathcal{Y}$, i.e., predictions are (non-empty) subsets of $\mathcal{Y}$, or more formally, $\sety\in 2^{\mathcal{Y}}\setminus \{\emptyset\}$.   

In the literature, different methods for set-valued prediction have been proposed (cf.\ Section \ref{sec:relwork}), essentially following two main directions. The first idea is to construct a set that covers the true outcome with a predefined (high) probability. A set-valued prediction of that kind can be seen as a generalization of the notion of confidence interval in frequentist statistics or credible interval in Bayesian statistics. A well-known representative of this statistical approach is conformal prediction \citep{Shafer2008}. The second direction is rooted in (Bayesian) decision theory and involves the notion of a \emph{utility function}, which evaluates a set-valued prediction in terms of its usefulness \citep{Delcoz2009LearningNC,Corani2008NCC,Corani2009LNCC,Zaffalon2012EvaluatingCC,Yang2017b}. Typically, the utility specifies a compromise between two natural though conflicting criteria: like in the statistical approach, the prediction should be \emph{correct} in the sense of covering the true class, but at the same time, it should be \emph{precise} and not contain too many options. Given a utility function of that kind, combined with a probability estimate on the classes, the natural decision-theoretic approach consists of predicting the set with highest expected utility. In this paper, we will focus on this approach, which we refer to as the \emph{set-based utility maximization framework}. 
 

\subsection{Set-based utility maximization}
 In set-based utility maximization, the quality of the prediction $\sety$ can be expressed by means of a set-based utility function $u(c,\sety)$, where $c$ corresponds to the ground-truth class and $\sety$ is the predicted set.  Typically, a decision-theoretic framework is considered, where one estimates a probabilistic model, followed by an inference procedure at prediction time.   At prediction time, the goal is to find the Bayes-optimal solution $\sety^{*}_u$ by expected utility maximization:
\begin{align}
\label{eq:bayesoptimal}
\sety^{*}_u(\bx) &= \argmax_{\sety \in 2^{\mathcal{Y}}\setminus \{\emptyset\}} \mathbb{E}_{P(c\,|\,\bx)} [u(c,\sety)]  = \argmax_{\sety \in 2^{\mathcal{Y}}\setminus \{\emptyset\}} \sum_{c \in \mathcal{Y}} u(c,\sety) P(c\,|\,\bx)\,, \nonumber \\
& =  \argmax_{\sety \in 2^{\mathcal{Y}}\setminus \{\emptyset\}} U(\sety,P,u) \,, 
\end{align}
where we introduce the shorthand notation $U(\sety,P,u)$ for the expected utility.  

Several authors have solved this optimization problem for different utility functions that are members of a general family $u: \mathcal{Y} \times 2^{\mathcal{Y}}\setminus \{\emptyset\} \rightarrow [0,1]$, defined as follows:
\begin{equation}
\label{eq:ufamily}
u(c,\sety) = \left \{ 
\begin{array}{ll}
0 &\quad \mbox{if $c \notin \sety$} \,, \\
g(|\sety|)&\quad \mbox{if $c \in \sety$} \,,
\end{array}
\right.
\end{equation}
where $|\sety|$ denotes the cardinality of the predicted set $\sety$. 
This family is characterized  by a decreasing sequence $(g(1), \ldots,g(K)) \in [0,1]^K$ that can have different forms.  
\citet{Delcoz2009LearningNC}, who use \emph{nondeterministic classification} as a synonym for set-based utility maximization, concentrate on three scores from the information retrieval community: precision with $g_P(s) = 1/s$, recall with $g_R(s)=1$, and the F$_{\beta}$-measure with $g_{F\beta}(s) = (1+\beta^2)/(\beta^2+s)$. Other utility functions with specific choices for $g$ are also studied in the literature \citep{Corani2008NCC,Corani2009LNCC,Zaffalon2012EvaluatingCC,Yang2017b,Nguyen2018RelMCC}. Those utility functions include:  
$$g_{\delta,\gamma}(s) = \frac{\delta}{s} - \frac{\gamma}{s^2} \,, \qquad g_{\delta}(s) = 1- \exp{\left(-\frac{\delta}{s}\right)},   \quad g_{\log}(s) = \log \left(1 + \frac{1}{s} \right) \,.$$
Especially $g_{\delta,\gamma}(s)$ is commonly used in the above papers, where $\delta$ and $\gamma$ can only take certain values to guarantee that the utility is in the interval $[0,1]$. Precision (also called discounted accuracy) corresponds to the case $(\delta,\gamma)=(1,0)$. However, typical choices for $(\delta,\gamma$) are $(1.6,0.6)$ and $(2.2,1.2)$ \citep{Nguyen2018RelMCC}, which overcome some of the limitations of precision (see below for a discussion). The utility function $g_{\delta}$ is an exponentiated version of precision, where the parameter $\delta$ controls the reward when sets become larger. 

\subsection{Contributions and outline}
In this paper we will focus on aspects related to optimizing (\ref{eq:bayesoptimal}). This is a non-trivial optimization problem,
as a brute-force search requires checking all subsets of $\mathcal{Y}$, resulting in an exponential time complexity. However, we will be able to find the Bayes-optimal prediction more efficiently. As our main contribution, we present several algorithms that solve (\ref{eq:bayesoptimal}) in an efficient manner. In the literature the work of \citet{Delcoz2009LearningNC} is the closest to our work. We extend their work in two directions: 1) the algorithms that we introduce are more efficient in multi-class classification settings where the number of classes is large, such as language modelling and reinforcement learning, and 2) our algorithms are applicable to a wide range of utility functions, unlike the algorithm of \citet{Delcoz2009LearningNC}, which concentrates on the F$_\beta$-measure. 


In Section~\ref{sec:theory} we present several theoretical results. Those results are essential building blocks for solving (\ref{eq:bayesoptimal}), but we also discuss the impact of these results for different utility functions. The algorithms that we develop are further materialized in Section~\ref{sec:algorithmic_solutions}. We first discuss an exact Bayes-optimal algorithm that makes $K$ queries to the conditional distribution $P(c\, |\, \bx)$, with $K$ the number of classes. In addition, we also introduce two approximate algorithms that make less than $K$ calls to $P(c \,|\, \bx)$. Those algorithms are based on two different paradigms: maximum inner product search and
hierarchical factorization of the conditional distribution. To conclude the theoretical part of this work, we provide an overview of related work in Section \ref{sec:relwork}. In Section~\ref{sec:res:exps} three different types of experimental results are discussed. In a first experiment we use image and text classification datasets to highlight the usefulness of set-valued prediction in case of uncertainty. In a second type of experiments, we evaluate the exact algorithm against some simple baselines for set-valued prediction.
In a last experiment, we focus on the runtime of the exact algorithm, highlighting the additional speedups that can be obtained by considering approximate algorithms for set-valued prediction.   

\section{Theoretical results}
\label{sec:theory}

In this section, we present several theoretical results as building blocks of the algorithms that we consider later on. We start with some general results, followed by a discussion of considerations for specific utility functions. 

\subsection{General results}
The formulation in (\ref{eq:bayesoptimal}) seems to suggest that all subsets of $\mathcal{Y}$ need to be analyzed to find the Bayes-optimal solution, but a first  result shows that this is not the case. 

\begin{restatable}{thm}{ksubsets}
\label{thm:ksubsets}
The exact solution of (\ref{eq:bayesoptimal}) can be computed by analyzing only $K$ subsets of $\mathcal{Y}$.  
\end{restatable}
\begin{proof}
With
$P(\sety\,|\,\bx) = \sum_{c \in \sety} P(c\,|\, \bx)$, the expected utility can be written as 

\begin{align}
\label{eq:topl}
U(\sety,P,u) &= \sum_{c \in \mathcal{Y}} u(c,\sety) P(c\,|\,\bx)
     = \sum_{c \in \sety} u(c,\sety) P(c\,|\,\bx)+
          \sum_{c' \notin \sety} u(c',\sety) P(c'\,|\,\bx)\,,
     \nonumber \\ 
&=  \sum_{c \in \sety} g(|\sety|) P(c\,|\,\bx) = g(|\sety|)  P(\sety\,|\,\bx) \,,
\end{align}
where the last summation in the second equality cancels out since $u(c',\sety)=0$. Let us decompose (\ref{eq:bayesoptimal}) into an inner and an outer maximization step.
The inner maximization step then becomes
\begin{align}
\sety_u^{*s} 
&= \argmax_{|\sety| = s} g(s)  P(\sety\,|\,\bx) = \argmax_{|\sety| = s}  P(\sety\,|\,\bx) \,,
\label{eq:innerexact}
\end{align}
for $s = \{1,\ldots,K\}$, where the last equality trivially holds due to $g(s)$ being constant. This step can be done very efficiently, by sorting the  conditional class probabilities, as for a given $s$, only the subset with highest probability mass needs to be considered. The outer maximization simply consists of computing
\begin{align}
\sety_u^{*}(\bx) =  \argmax_{\sety\in \{\sety_u^{*1},\ldots,\sety_u^{*K}\}} g(|\sety|)  P(\sety\,|\,\bx) \,, 
\label{eq:outerexact}
\end{align}
which only requires the evaluation of $K$ sets. 
\end{proof} 

So, one only needs to evaluate $\sety_u^{*1},\ldots,\sety_u^{*K}$ to find the Bayes-optimal solution, which limits the search to $K$ subsets. 
In fact, we can even do better as it turns out that by restricting $g$, we can assure that the sequence $U(\sety_u^{*1},P,u),\ldots,U(\sety_u^{*K},P,u)$ will have reached its global maximum when it starts to decrease. This will allow us to further limit the search, by means of an early stopping criterion, as soon as we reach that maximum. The restriction required for $g$ is $(1/x)$-convexity, i.e., convexity after a $(1/x)$ transformation. This is a somewhat surprising and rather technical result that is summarized in the following definition and theorem. 

\begin{restatable}{defn}{defOneOverXConcave}
A sequence $g(1),g(2),\ldots,g(K)$ is $(1/x)$-convex if 
\begin{equation}
\label{eq:1overx}
1/g(s+1) \leq \frac{1/g(s) + 1/g(s+2)}{2} \quad \mbox{for all } s \in \{1,\ldots,K-2\} \, .
\end{equation}
\end{restatable}

\begin{restatable}{thm}{stopping} 
Let $g(1),g(2),\ldots,g(K)$ be a decreasing $(1/x)$-convex sequence. Then the following implication holds for any $s \in \{1,\ldots,K-2\}$: 
$$U(\sety_u^{*s},P,u) > U(\sety_u^{*s+1},P,u) \Longrightarrow U(\sety_u^{*s+1},P,u) > U(\sety_u^{*s+2},P,u) \,.$$ 
\label{thm:stopping}
\end{restatable}
\begin{proof}
Let us first observe the following equivalence:
\begin{eqnarray}
\label{eq:semi}
&&1/g(s+1) \leq \frac{1/g(s) + 1/g(s+2)}{2} \nonumber \\
&\Leftrightarrow & 2 \leq \frac{g(s+1)}{g(s)} + \frac{g(s+1)}{g(s+2)} \nonumber \\ 
&\Leftrightarrow &g(s)g(s+1) + g(s+2) g(s+1) \geq 2g(s) g(s+2) \nonumber \\
&\Leftrightarrow& g(s)[g(s+1) - g(s+2)] \geq g(s+2) [g(s) - g(s+1)] \nonumber \\
&\Leftrightarrow& \frac{g(s)}{g(s) - g(s+1)} \geq \frac{g(s+2)}{g(s+1) - g(s+2)} \nonumber \\
&\Leftrightarrow& \frac{g(s+1)}{g(s) - g(s+1)} + 1 \geq \frac{g(s+2)}{g(s+1) - g(s+2)}\,. 
\end{eqnarray}
Assume that for a given $s$ it holds that $U(\sety_u^{*s},P,u) > U(\sety_u^{*s+1},P,u)$. Let $p_i = P(c_i \, | \, \bx)$ and observe that the following equivalences hold: 
\begin{eqnarray}
U(\sety_u^{*s},P,u) > U(\sety_u^{*s+1},P,u) &\Leftrightarrow& g(s) \sum_{i=1}^s p_i > g(s+1) \sum_{i=1}^{s+1} p_i \nonumber \\
&\Leftrightarrow& [g(s)-g(s+1)] \sum_{i=1}^s p_i > g(s+1) p_{s+1}  \nonumber \\
&\Leftrightarrow&  \sum_{i=1}^s p_i > \frac{g(s+1)}{g(s)-g(s+1)} p_{s+1} \,.
\label{eq:s}
\end{eqnarray}
Observe that from (\ref{eq:semi}) and (\ref{eq:s}) it follows that:
\begin{eqnarray*}  \sum_{i=1}^{s+1} p_i > \frac{g(s+1)}{g(s)-g(s+1)} p_{s+1} + p_{s+1} 
&\Leftrightarrow&  \sum_{i=1}^{s+1} p_i > \left(\frac{g(s+1)}{g(s)-g(s+1)} +1 \right) p_{s+1} \\
&\Rightarrow&  \sum_{i=1}^{s+1} p_i > \frac{g(s+2)}{g(s+1)-g(s+2)} p_{s+2} \\
&\Leftrightarrow&U(\sety_u^{*s+1},P,u) > U(\sety_u^{*s+2},P,u) \,.
\end{eqnarray*}
This is what we needed to prove. 
 \end{proof}

Thus, what Theorem~\ref{thm:stopping} tells us is that, for $(1/x)$-convex sequences, we have found a stopping criterion so that even less than $K$ sets need to be analyzed when optimizing (\ref{eq:bayesoptimal}). More specifically, we can stop as soon as the sequence 
$$
U(\sety_u^{*1},P,u),U(\sety_u^{*2},P,u)\ldots,U(\sety_u^{*K},P,u)
$$
starts to decrease. The stopping criterion is $U(\sety_u^{*s},P,u) > U(\sety_u^{*s+1},P,u)$ for a given $s \in \{1,\ldots,K-1\}$.

Theorems \ref{thm:ksubsets} and \ref{thm:stopping} provide guarantees to optimize (\ref{eq:bayesoptimal}) in a classical decision-theoretic context. 
In the appendix, we present a short theoretical analysis that relates the Bayes-optimal solution for the set-based utility functions to the solution obtained on the conditional class probabilities given by a trained model. The goal is to upper bound the regret of the set-based utility functions by the $L_1$ error of the class probability estimates. Similarly as in decision-theoretic utility maximization, the analysis is performed on the level of a single $\bx$. 
 
\subsection{Considerations w.r.t. specific utility functions}
Let us discuss the implications of the above theorems for the utility functions that were mentioned in the introduction (see also Table~\ref{tbl:utility-scores} and the left panel of Figure~\ref{fig:g} for an overview). Is $(1/x)$-convexity satisfied for those utility functions? For the ones that are most commonly used in the literature the answer turns out to be yes: precision, recall, the F$_\beta$-measure, as well as the $g_{\delta,\gamma}$ family for recommended values of $\delta$ and $\gamma$, are all utilities with associated $(1/x)$-convex sequences. Let us remark that $(1/x)$-convexity cannot easily be assessed by plotting the graph of a specific sequence $g(s)$.  In practice one needs to check this formally using (\ref{eq:1overx}) for all $s$. 

 Precision, with $g_P(s) = 1/s$, in fact defines ``how convex'' a sequence is allowed to be, because (\ref{eq:1overx}) is satisfied as an equality in that boundary case. As shown in the left panel of Figure~\ref{fig:g}, most utility functions from the literature behave very similar to precision; they decrease very quickly, and their curvature is similar to precision, but they are somewhat less convex (remark that for functions the degree of convexity is determined by the slope of the first derivative). From that perspective, it is obvious that concave sequences will be $(1/x)$-convex. The following proposition states this formally.  

\begin{defn}
A sequence $g(1),g(2),...,g(K)$ is concave if 
$$g(s+1) \geq \frac{g(s) + g(s+2)}{2} \qquad \mbox{for all } s \in \{1,...,K-2\}$$
\end{defn}

\begin{proposition}
\label{thm:concave}
Let $g(1),g(2),...,g(K)$ be a decreasing, concave sequence. Then $g(1),g(2),...,g(K)$ is also $(1/x)$-convex.  
\end{proposition}
\begin{proof}
As shown above, the following equivalence holds: 
\begin{eqnarray}
\label{eq:equi}
&&1/g(s+1) \leq \frac{1/g(s) + 1/g(s+2)}{2} \nonumber \\
&\Leftrightarrow& g(s)[g(s+1) - g(s+2)] \geq g(s+2) [g(s) - g(s+1)] 
\end{eqnarray}
Due to the fact that $g(1),g(2),...,g(K)$ is a decreasing, concave sequence we know that $g(s) \geq g(s+2) $ and $g(s+1) - g(s+2) \geq g(s) - g(s+1) \,.$
Combining the two inequalities lets us conclude that the sequence must be $(1/x)$-convex.
\end{proof}

Apart from $(1/x)$-convexity, there is another property that an interesting utility function should obey: $g(s)$ should be lower bounded by precision, i.e.\ $g_P$. Precision and recall are frequently used in binary classification, but one may argue that they are both not very useful utility functions for assessing set-valued predictions. For recall this is pretty obvious, as this measure does not have any penalty for the size of the set. Yet, also for precision, there is a problem. Its utility maximiser will always be a set of cardinality one. For example, consider a multi-class problem with hundred classes, and assume that for a given instance the conditional class probabilities are $0.1$ for ten of these hundred classes. Clearly, in this case, the best prediction is to return a set consisting of the ten classes, resulting in an expected utility of $0.1$. If $ g(10) = 1/10$, a singleton set that contains one of the ten classes also yields an expected utility of $0.1$. Both are Bayes-optimal predictions in view of (\ref{eq:bayesoptimal}). Thus, the problem with precision is that it is not risk-averse. In the face of uncertainty, a risk-averse utility function will prefer a set with many classes over a singleton set that contains one of those classes \citep{Zaffalon2012EvaluatingCC}. 
That is why we highlight $g(s) \geq 1/s$ as an absolute requirement. Utility functions violating this requirement are practically pointless, because the solution to (\ref{eq:bayesoptimal}) is always a set of cardinality one in that case. The next proposition formalizes this claim, which has been reported by \citet{Zaffalon2012EvaluatingCC} from a different perspective and using a different notation.  

\begin{proposition}
\label{prop:precision}
Let $g$ be a sequence such that $g(1) = 1$, then for any distribution $P$ the following statements hold.
\begin{itemize}
\item When $g(s) < 1/s$ for some $s >1$, then $\sety_u^{*s}$ is not a solution of (\ref{eq:bayesoptimal}), 
\item When $g(s) < 1/s$ for all $s > 1$, then the unique solution of (\ref{eq:bayesoptimal}) is $\sety_u^{*1}$,
\item  When $g(s) = 1/s$ for all $s$, then $\sety_u^{*1}$ is a solution of (\ref{eq:bayesoptimal}). 
\end{itemize}
\end{proposition}
\begin{proof}
Let us start by proving the first statement. When $g(s) < 1/s$ it follows that:
$$U(\sety_u^{*s},P,u) < \frac{P(\sety_u^{*s}\,|\,\bx)}{s} \leq P(\sety_u^{*1}\,|\,\bx) = U(\sety_u^{*1},P,u) \,,$$
where the second inequality holds because $\sety_u^{*1}$ must correspond to the mode of $P$, and the last equality holds because we assume $g(1)=1$. So, we have shown that, when $g(s) < 1/s$ for some $s$, it holds that $ U(\sety_u^{*s},P,u) < U(\sety_u^{*1},P,u)$ for any distribution $P$.
This reasoning can be applied to any $s>1$, which also proves the second statement. Using $g(s) = 1/s$ in the above reasoning, the inequalities become equalities. This proves the third statement. 
\end{proof}

The proposition lets us conclude that it is pointless to use any utility for which $g(s) < 1/s$ for all $s$, because $\sety_u^{*1}$ is then the unique solution of (\ref{eq:bayesoptimal}). It is also pointless to use $g_P$ as utility function: $\sety_u^{*1}$ will then be one of the solutions of (\ref{eq:bayesoptimal}), but there might be other solutions as well. When $g(s) < 1/s$ for some but not all $s$, only sets of specific cardinalities can be the solution of (\ref{eq:bayesoptimal}). This is probably also unwanted in practice. 

So, what about the other utility functions from the literature? In Figure~\ref{fig:g}, one can see that utility functions $g_{F1}$ and $g_{\delta=1.6,\gamma=0.6}$ behave similarly as $g_P$, but they are lower bounded by $g_P$, so with those utility functions it will be possible to predict non-singleton sets. In general, the faster those functions decrease, the smaller the sets that will be predicted. For $g_{F\beta}$ the parameter $\beta$ controls the degree of convexity, as shown in the right panel of Figure~\ref{fig:g}. As such, this parameter will in a way also control the size of the sets that are predicted. Note that Proposition~\ref{prop:precision} cannot be applied for $g_{\log}$ and $g_{\delta}$, because $g(1) \neq 1$ in those cases. One would need to rescale those utility functions first, before applying the proposition. 

As a summary of the above discussion, we define four properties that an interesting utility function $g$ should have:
\begin{itemize}
\item $g$ should be strictly decreasing. Smaller sets should have a bigger utility than bigger sets.
\item $g(1) = 1$. This is just an interesting property to compare different utility functions. When a function violates this property, it is best rescaled. 
\item $g(s) >1/s$ for all $s$. The spectrum of utility maximizers will be limited to sets of particular cardinalities when $g(s) <1/s$ for some $s$. If $g(s) <1/s$ for all $s$, the utility function becomes completely useless.   
\item $g$ should be $(1/x)$-convex. This guarentees that the utility maximizer can be found efficiently.     
\end{itemize}
There is a close link between the third and fourth property: if $g(s) < 1/s$ for some $s$, then $g$ cannot be $(1/x)$-convex. However, the converse is not true. To give such an example, let us introduce a novel family of utility functions that generalizes a utility function that is often used in the literature on multi-class classification with reject option, see e.g.\ \citep{Ramaswamy2015CAMCRO}. 
When a reject option is allowed, the prediction can only be a singleton or the full set $\mathcal{Y}$ containing $K$ classes. The first case typically gets a reward of one, while the second case should receive a lower reward, e.g.\ $1-\alpha$.  This second case corresponds to abstaining, i.e., not predicting any class label, and the (user-defined) parameter $\alpha$ specifies a penalty for doing so, with the requirement $0< \alpha < 1-1/K$ to be risk-averse. To include sets of any cardinality $s$, the utility could be generalized as follows:
\begin{equation}
\label{eq:ualphabeta}
g_{\alpha,\beta}(s) = 1 - \alpha \left(\frac{s-1}{K-1} \right)^\beta \,,
\end{equation}
which we call the generalized reject option utility. Here, we have the same interpretation for $\alpha$, whereas $\beta \in (0,\infty)$ defines whether $g(s)$ is convex or concave, as shown in the bottom panel of Figure~\ref{fig:g}. While convexity (like in most of the above utility functions) appears natural in most applications, a concave utility might be useful when predicting a large set is tolerable. In the limit, when $\beta \rightarrow 0$, we obtain the simple utility function for classification with reject option. 

The $g_{\alpha,\beta}$ family is quite intuitive from an application perspective, and it has a lot of flexibility. This makes this family interesting, but for certain parameterizations it is not lower bounded by precision. It is important to choose $\alpha \leq \frac{K-1}{K}$ and $\beta \geq \log_{\frac{1}{K-1}}\frac{K}{2}+1$ to guarantee that $g_{\alpha,\beta}(s) \geq 1/s$ for all $s$ (see appendix for derivations). For example, as shown in the figure in the appendix, for $\alpha = 1$, $g_{\alpha,\beta}(s)$ is dominated by $g_P(s)$ for large $s$. We also observed that $g_{\alpha,\beta}(s)$ is for some $\alpha$ and $\beta$ not $(1/x)$-convex. In contrast, the $g_{F\beta}$ family and $g_{\delta,\gamma}$ for recommended values of $\delta$ and $\gamma$ deliver utility functions that satisfy all the four properties that are listed above. This makes them interesting from an application perspective, and they will be our focus in the experiments.   
\begin{table}[t]
\centering
\caption{The utility functions $g(s)$ discussed in this paper.}
\label{tbl:utility-scores}
\arraycolsep=1pt
\begin{scriptsize}
\begin{tabular}{lll}
\toprule
Name & $g(s)$ & Article \\
\midrule
Precision & 
$g_P(s) = \frac{1}{s}$ &
\citep{Delcoz2009LearningNC}\\
Recall & 
$g_R(s) =1$ & 
\citep{Delcoz2009LearningNC}\\
F$_{\beta}$-measure &
$g_{F\beta}(s) =\frac{1+\beta^{2}}{s+\beta^{2}}$ &
\citep{Delcoz2009LearningNC}\\
Credal Utility &
$g_{\delta,\gamma}(s) =\frac{\delta}{s} - \frac{\gamma}{s^2}$ &
\citep{Zaffalon2012EvaluatingCC} \\
Exp. Utility &
$ g_{\delta}(s) =1-\exp{\left(-\frac{\delta}{s}\right)}$ &
\citep{Zaffalon2012EvaluatingCC} \\
Log. Utility &
$ g_{\log}(s) =\log(1 + \frac{1}{s})$ &
\citep{Zaffalon2012EvaluatingCC} \\
Reject Option & 
$g_{\mathrm{rej}}(s) =\left \{ 
\begin{array}{ll}
1 & \mbox{if $s=1$} \,, \\
1-\alpha & \mbox{if $s=K$} \,.
\end{array}
\right.$ 
&
\citep{Ramaswamy2015CAMCRO} \\
Gen. Reject Option & 
$ g_{\alpha,\beta}(s) = 1 - \alpha \left(\frac{s-1}{K-1} \right)^\beta$ &
Extension of \citep{Ramaswamy2015CAMCRO} \\
\bottomrule
\end{tabular} 
\end{scriptsize}
\end{table}
\begin{figure}[ht]
\centering
\begin{tabular}{cc}
\hskip -0.2in
\includegraphics[scale=0.35]{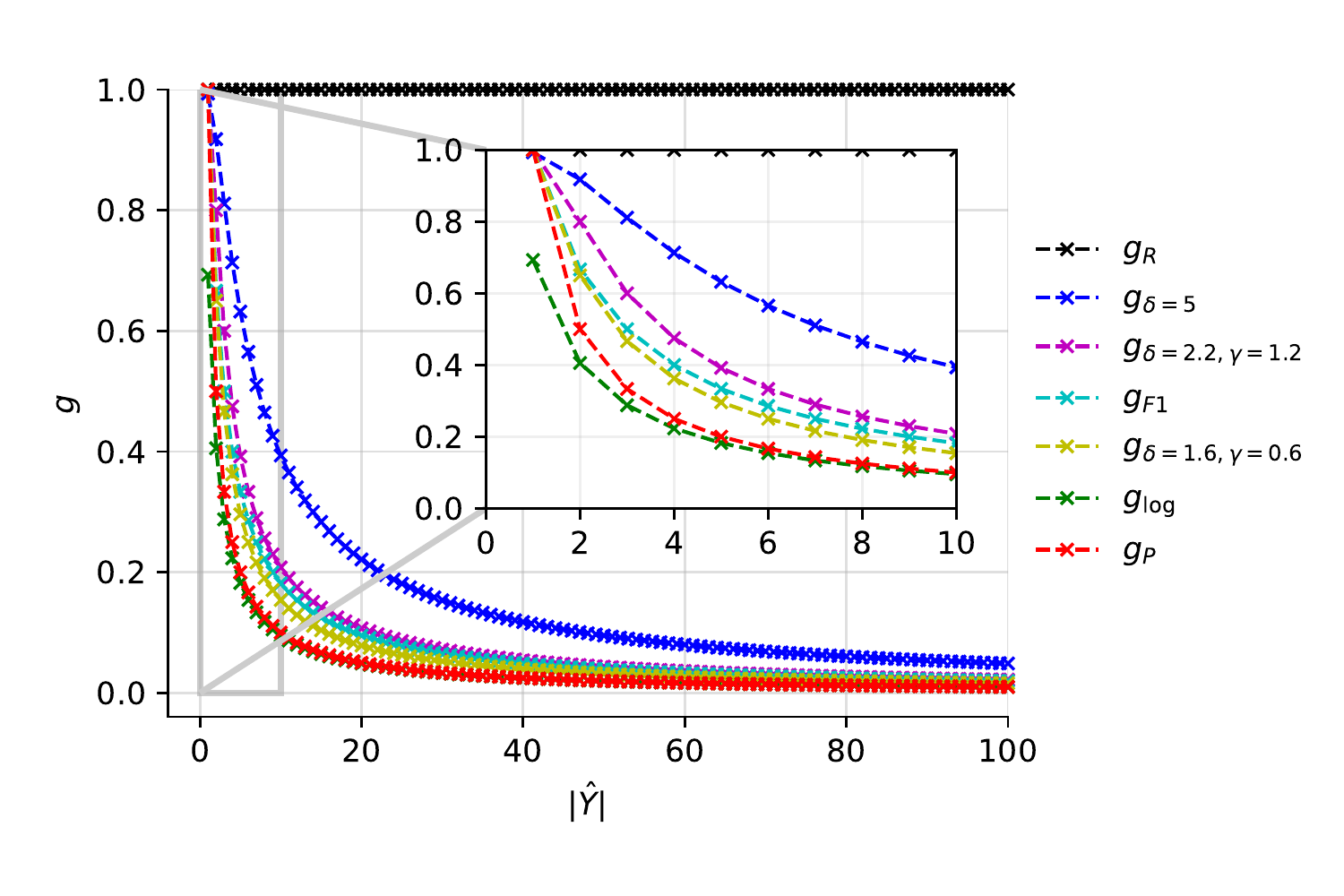} &
\includegraphics[scale=0.35]{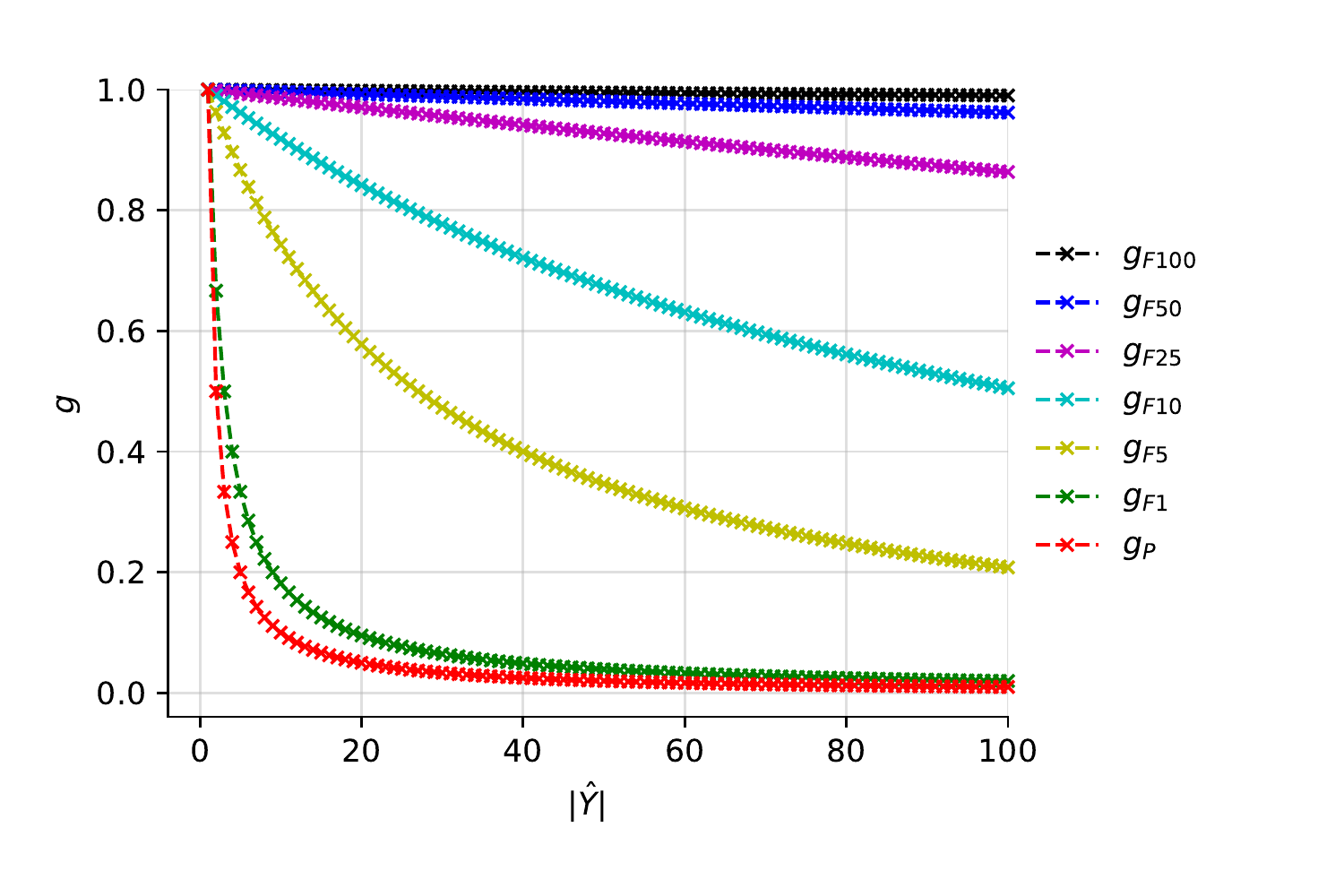} 
\end{tabular}
\caption{A visualization of $g$ in function of different values of $|\sety|$ and set-based utility functions. $K=100$.}
\label{fig:g}
\end{figure}

\section{Algorithmic solutions}
\label{sec:algorithmic_solutions}

\begin{algorithm}[t]
\begin{small}
\caption{SVBOP -- \textbf{input:} $u(\cdot)$, $\bx$, $\calY = \{c_1, \ldots, c_K\}$, $PC$  }
\begin{algorithmic}[1] 
\State $\sety \leftarrow \emptyset$, $p_{\sety} \leftarrow 0$, $U^* \leftarrow 0$ \Comment{{\scriptsize Initialize the current best solution, its probability and utility}}
\State $PC$.initPrediction($x$, $\calY$) \Comment{{\scriptsize Initialize the prediction procedure}}
\While{$(c, p_c) \leftarrow$  $PC$.getNextClass()}  \Comment{{\scriptsize Repeat until all the classes are returned by $PC$}}
	\State $\sety \leftarrow \sety \cup \{c\}$, $p_{\sety} \leftarrow p_{\sety} + p_c$ \Comment{{\scriptsize Update the current solution and its probability}}
	\State $U_{\sety} \leftarrow p_{\sety} \times g(\sety)$ \Comment{{\scriptsize Compute $U(\sety,P,u)$ according to Eq.~(\ref{eq:topl})}}
	\If{$U^* \leq U_{\sety}$} \Comment{{\scriptsize If the current solution is better than the best solution so far}}
		\State $\sety^{*}_u \leftarrow \sety$, $U^* \leftarrow U_{\sety}$ \Comment{{\scriptsize Replace the current best solution}}
		    \Else
	        \textbf{ break} \Comment{{\scriptsize If there is no improvement break the while loop according to Theorem \ref{thm:stopping}}}
        \EndIf
\EndWhile
\State \textbf{return} $\sety^{*}_u$ \Comment{{\scriptsize Return the set of classes with the highest utility}}
\end{algorithmic}
\label{alg:ubop-framework}
\end{small}
\end{algorithm}

The above theoretical results (in particular Theorems~\ref{thm:ksubsets} and~\ref{thm:stopping}) 
suggest that problem (\ref{eq:bayesoptimal}) can be efficiently solved for ($1/x$)-convex set-based utility functions. 
In this section, we present three algorithmic solutions for this problem, and are all based on the same generic framework. 
The first algorithm returns the Bayes-optimal solution to (\ref{eq:bayesoptimal}) in an exact manner. 
The other two algorithms compute approximate solutions, but yield substantial runtime gains. 
Those algorithms can be used when the number of classes is large. 

Algorithm~\ref{alg:ubop-framework} presents the generic framework. We use the acronym SVBOP, which stands for Set-Valued Bayes-Optimal Prediction. SVBOP uses a probabilistic classifier, denoted as $PC$, that supports two operations. 
The first operation, \textit{initPrediction}, initializes the prediction procedure for a given test example $\bx$. 
The second operation, \textit{getNextClass}, placed in the while loop, returns the next class label with respect to decreasing conditional class probabilities. 
In each subsequent iteration of the while loop, 
the solution of inner maximization~(\ref{eq:innerexact}) is for a given $s$ found 
by adding the class with the $s$-highest conditional class probability to the predicted set. 
In this way, $U(\sety_u^{*1},P,u)$,\ldots,$U(\sety_u^{*s},P,u)$ are computed in a sequential way, 
till the stopping criterion of Theorem~\ref{thm:stopping} is satisfied.

This framework can be seen as a generalization of an algorithm introduced by \citet{Delcoz2009LearningNC} 
for optimizing the F$_\beta$-measure in multi-class classification. 
There is also a strong correspondence with certain F$_\beta$-maximization algorithms in multi-label classification 
(see e.g.~\citealt{jansche07,Ye2012,Waegeman2014}). 


\subsection{SVBOP-Full}
\label{sec:ubop-ovr}

\begin{algorithm}[t]
\begin{small}
\caption{SVBOP-Full.initPrediction -- \textbf{input:} $\bx$, $\calY = \{c_1, \ldots, c_K\}$}
\begin{algorithmic}[1] 
\State $\calQ \leftarrow \emptyset$ \Comment{{\scriptsize Initialize a list to store classes and $P(c \,|\,  \bx)$}} 
\For{$c \in \calY$}
  $p_c \leftarrow P(c \,|\, \bx)$, $\calQ$.add\big(($c$, $p_c$)\big) \Comment{{\scriptsize Query $PC$ to get $P(c \,|\, \bx)$ for all classes}}
\EndFor
\State $\calQ$.sort()  \Comment{{\scriptsize Sort the list decreasingly according to $P(c \,|\, \bx)$}}
\end{algorithmic}
\label{alg:ubop-ovr-init}
\end{small}
\end{algorithm}

\begin{algorithm}[t]
\begin{small}
\caption{SVBOP-Full.getNextClass}
\begin{algorithmic}[1] 
\State \textbf{return} $\calQ$.pop() \Comment{{\scriptsize Pop the next highest element from the sorted list.}} 
\end{algorithmic}
\label{alg:ubop-ovr-next}
\end{small}
\end{algorithm}

The first algorithm is further referred to as SVBOP-Full, because it computes all conditional class probabilities. $PC$ is here a standard multi-class probabilistic classifier
that returns the estimated conditional class distribution for a given test example $\bx$. Examples of such classifiers are logistic regression, linear discriminant analysis, gradient boosting trees or neural networks with a softmax output layer.
%
The inference algorithm starts by querying the underling classifier to get all $K$ conditional class probabilities $P(c\,|\,\bx)$. 
Then, the conditional class probabilities are sorted in decreasing order (Algorithm~\ref{alg:ubop-ovr-init}). 
When in Algorithm~\ref{alg:ubop-framework} the next class label is needed, it can simply be taken from this sorted list (Algorithm~\ref{alg:ubop-ovr-next}).

This approach is simple and elegant but requires sorting of all $K$ conditional class probabilities, which results in an $O(K \log K)$ complexity. 
However, the most costly operation is usually querying the distribution $P$ to obtain values of conditional probabilities for all $K$ classes. 
In case of linear models, this cost scales linearly with the number of classes, multiplied by the number of non-zero feature values. 
For problems with a large number of classes, often referred to as extreme classification problems~\citep{Prabhu_Varma_2014}, this is usually too costly.

\subsection{Hierarchical search with similarity graphs (SVBOP-HSG)}


Since only class labels with high probability mass are required, a procedure would be desirable that  returns the top classes without the need to compute conditional class probabilities for all classes. 
To accomplish this, we leverage approaches for approximate nearest neighbor search \citep{Yagnik_et_al_2011, Shrivastava_Ping_2014, Johnson_et_al_2017} and adapt them to our setting. The use of such methods essentially becomes possible through two problem transformations: first, we reduce maximum probability search to maximum inner product search, which we then in turn  reduce to nearest neighbor search. In the following, we briefly comment on both reduction steps and eventually on approximate nearest neighbor search itself. 

As for the first step, note that most learning algorithms produce class probability estimates in a last step of the learning procedure by mapping scores to probabilities. A typical example, which is also used in our approach, is the softmax transformation:

\begin{equation}\label{eq:softmax}
P(c \,\vert\, \bx) = \frac{\exp(\bw_c \cdot \bx)}{\sum_{c' \in \calY} \exp(\bw_{c'} \cdot \bx )} \, .
\end{equation}

Obviously, as long as the probability is a monotone function of an inner product $\bw_{c} \cdot \bx$ between the query instance $\bx$ and a weight vector $\bw_{c}$, like in (\ref{eq:softmax}), finding $\argmax_{c}P(c \,\vert\, \bx)$ is equivalent to finding  $\argmax_{c} \bw_{c} \cdot \bx$. More generally, finding the top-$s$ inner products corresponds to finding the top-$s$ probabilities. 
Futhermore, there is no need to compute the value of the partition function, i.e., the denominator of (\ref{eq:softmax}), 
to find the optimal set-valued prediction. 
For a given $\bx$, the value of the partition function is constant for all $c$ 
and since $\argmax f(\bx) = \argmax a \times f(\bx)$, 
for any constant $a > 0$, 
the lack of normalization does not affect the result of both 
inner~(\ref{eq:innerexact}) and outer~(\ref{eq:outerexact}) maximization.
%


As for the second step, let us assume that, as discussed before, we have a linear model for each class $c$, which is represented by a vector $\bw_c$ in a suitable (perhaps transformed) feature space $\mathcal{X}$ (e.g., the last but one layer in a neural network, which is mapped to the output via softmax). Now, since the squared Euclidean distance between $\bw_c$ and $\bx$ is given by $d(\bw_c, \bx) = \| \bw_c \|_2 - 2  \bw_c \cdot \bx + \| \bx \|_2$, maximizing  $\bw_c \cdot \bx$ is ``almost'' equivalent to minimizing $d(\bw_c, \bx)$. More specifically, it is equivalent to minimizing the distance $d(\bw_c', \bx')$ between two expanded vectors $\bw_c'$ and $\bx'$. The former is obtained by adding an entry $-\sqrt{\| \bw_c \|_2}$ to $\bw_c$, and the latter by adding a $0$ to $\bx$. Consequently, maximizing inner products in $\mathcal{X}$ is equivalent to minimizing distances in the augmented space $\mathcal{X}'$.\footnote{Practically, this augmentation is often omitted for simplicity.} 


A reduction, as outlined above, is interesting because many methods for efficient nearest neighbor search have been proposed in the literature. In this work, we use Hierarchical Navigable Small World (H-NSW) graphs, introduced by~\citet{Malkov_and_Yashunin_2018}. This method is based on the concept of a similarity graph~\citep{Navarro_2002}, in which edges connect similar objects. In our case, these objects are weight vectors $\bw_{c}$. H-NSW uses multiple such graphs, each on a different level. The lowest level contains all objects, while higher levels contain successively sparser subsets of these objects. Roughly speaking, the idea is to search the nearest neighbors of the query $\bx$ level-wise (by traversing the graph for the layer in a greedy way), starting with the highest level. The neighbors found on each level do not necessarily correspond to the true neighbors of $\bx$ on the lowest level, but should at least indicate the region in which these neighbors are located, and hence provide a good entry point for refining the search on the next level. For technical details and the complete pseudocode of the H-NSW method, we refer to~\citet{Malkov_and_Yashunin_2018}. What the algorithm eventually returns is a list of $s$ weight vectors $\bw_c$ for the $s$ classes that have (approximately) the highest inner products with the test example $\bx$.


We conclude this section with two remarks. First, finding the top-$s$ classes may not be enough to satisfy the stopping criterion of Theorem \ref{thm:stopping}. To solve this problem, we use a simple doubling strategy. When $PC$ is requested to provide the $(s+1)$-st class, we double the value of $s$ and query the H-NSW index again. Since the nearest neighbor search is approximate, we append all new labels to the original list. In this way, we do not omit any new label with a probability higher than the minimum probability of labels found in the previous query. Second, this search method should lead to a faster inference than SVBOP-Full, as the number of inner products should be lower than a number required to compute all conditional class probabilities. While this method significantly speeds up inference, it adds additional cost to the training phase, due to the need to construct the H-NSW index. For training, one also relies on multi-class classifiers that scale linearly with the number of classes.

\begin{algorithm}[t]
\begin{small}
\caption{SVBOP-HSG.initPrediction -- \textbf{input:} $\bx$, $\calY = \{c_1, \ldots, c_K\}$}
\begin{algorithmic}[1] 
\State $\calQ \leftarrow \emptyset$ \Comment{{\scriptsize Initialize a list of classes with their inner prod. $\bw_{c} \cdot \bx $}} 
\State $i \leftarrow 0$ \Comment{{\scriptsize Initialize the class counter}} 
\State $\calQ' \leftarrow$ H-NSW Index.query($\bx$, $k$) \Comment{{\scriptsize Query for initial top-$k$ elements}}
\For{$c \in \calQ'$} \Comment{{\scriptsize For initial top-$k$ classes}}
    \State $\calQ\mathrm{.add}((c, \exp(\bw_c \cdot \bx)))$  \Comment{{\scriptsize Transform inner prod. to unnorm. $P(c|\bx)$ and add to list}}
\EndFor
\end{algorithmic}
\label{alg:ubop-hsg-init}
\end{small}
\end{algorithm}

\begin{algorithm}[t]
\begin{small}
\caption{SVBOP-HSG.getNextClass}
\begin{algorithmic}[1] 
    \State $i \leftarrow i + 1$ \Comment{{\scriptsize Increment class counter}}
	\If{$|\calQ| < i$} \Comment{{\scriptsize If class counter is greater than size of list of classes}}
	    \State $\calQ' \leftarrow$ H-NSW Index.query($\bx$, $2 \times
	    |\calQ|$) \Comment{{\scriptsize Perform doubling}}
	    \For{$c \in \calQ' / \calQ$} \Comment{{\scriptsize For each new class found by index}} 
	    \State $\calQ\mathrm{.add}((c, \exp(\bw_c \cdot \bx)))$  \Comment{{\scriptsize Transform inner prod. to unnorm. $P(c|\bx)$ and add to list}}
	    \EndFor 
    \EndIf
	\State \textbf{return} $\calQ$.at($i$) \Comment{{\scriptsize Return next class from the list}}
\end{algorithmic}
\label{alg:ubop-hsg-next}
\end{small}
\end{algorithm}

\subsection{Hierarchical factorization of the conditional distribution (SVBOP-HF)}

To have a compatible probabilistic classifier that allows for efficient prediction of top-$s$ classes, while having a much faster training at the same time, 
we investigate in this subsection a solution based on a hierarchical factorization of the distribution $P(c \,|\,\bx)$. 
This approach underlies many popular algorithms, such as nested dichotomies \citep{Fox_1997,Frank2004NestedD,melnikov2018}, 
conditional probability estimation trees \citep{Beygelzimer2009CondPT}, 
probabilistic classifier trees~\citep{Dembczynski2016ConsistencyOP},
or hierarchical softmax~\citep{Morin2005HierS}, often used in neural networks as an output layer. 

With a hierarchical tree structure over the classes, 
where the root represents the class space and leafs the singleton sets of classes, 
one can express the conditional class probability $P(c \,|\,\bx)$ via the chain rule of probability:
\begin{equation}
\label{eq:recpath}
P(c \, \vert \, \bx)  = \prod_{v \in \mathrm{Path}(c)} P(v \,|\, \mathrm{Parent}(v), \bx) \,,
\end{equation}
where $\mathrm{Path}(c)$ is a set of nodes on the path connecting the leaf and the root of the tree structure. 
$\mathrm{Parent}(v)$ gives the parent of node $v$, and for the root node $r$ we have $P(r \,|\, \mathrm{Parent}(r), \bx) = 1$. 
In each node of the tree, we train a multi-class probabilistic classifier of choice. 

For inference, we adapt an $A^*$-style algorithm, 
closely related to search methods used with probabilistic tree classifiers~\citep{Demb2012b,Dembczynski2016ConsistencyOP,Mena2016}. 
It uses a priority queue for storing visited nodes in the order of their decreasing conditional class probabilities. 
The queue is initialized with the root node (Algorithm~\ref{alg:ubop-hf-init}).
In the main loop, for each iteration, the next label is returned in order of decreasing conditional class probabilities. 
This is achieved by visiting nodes one by one, taking them from the queue and adding for each visited node its children to the queue (Algorithm~\ref{alg:ubop-hf-next}). 

On average, this search should result in significantly faster inference than the standard SVBOP-Full algorithm, as only a part of the tree will be explored. Optimistically, the speedup can be even exponential (i.e., the query for a single $\bx$ can proceed in logarithmic time in the number of classes $K$), 
Yet, in the worst case, the algorithm can visit all nodes in the tree, a number that is upper bounded by $2K-1$. With specific optimization algorithms, such as the ones used for hierarchical softmax implementations in deep neural networks, the hierarchical factorization might also lead to much faster training. One only needs to update nodes on a path from the root to a leaf corresponding to the class label of the example. This results in logarithmic training times in terms of the number of classes, assuming that the tree is balanced.

Similarly, as in the previous approach, there is an additional step required for building the hierarchical structure before training. This structure can be obtained from data. For example, 
Huffman trees are commonly used for similar algorithms in natural language processing problems~\citep{Mikolov_et_al_2013}.  More involved learning algorithms, such as the one used in~\citep{Prabhu_et_al_2018} run hierarchical balanced 2-means on class profiles. 
In some applications, a natural hierarchy may exist and this one can be used as well, as we will show in the experiments. 

\begin{algorithm}[t]
\begin{small}
\caption{SVBOP-HF.initPrediction -- \textbf{input:} $\bx$, $\calY = \{c_1, \ldots, c_K\}$}
\begin{algorithmic}[1] 
\State $\calQ = \emptyset$ \Comment{{\scriptsize Initialize a priority queue}}
\State $\calQ\mathrm{.add}((v_{\text{root}}, P(v_{\text{root}}|\bx))$ \Comment{{\scriptsize Add the tree root with the corresponding $P(v|\bx)$}}
\end{algorithmic}
\label{alg:ubop-hf-init}
\end{small}
\end{algorithm}

\begin{algorithm}[t]
\begin{small}
\caption{SVBOP-HF.getNextClass}
\begin{algorithmic}[1] 
\While{$\calQ \neq \emptyset$} \Comment{While the number of predicted labels is less than $k$}
	\State $(v, p_v) \leftarrow \calQ$.pop() \Comment{{\scriptsize Pop the node with highest $P(v|\bx)$ from the queue}}
	\If{$v$ is a leaf} \Comment{{\scriptsize If the node is a leaf node}} 
		\State \textbf{return} (Class($v$), $p_v$) \Comment{{\scriptsize Return corresponding class and $P(c|\bx)$}}
    \EndIf
	\For{$v' \in $ Children($v$)} \Comment{{\scriptsize Else for all child nodes of $v$}}
	    \State $p_{v'} \leftarrow p_v \times P(v'|v,\bx)$ \Comment{{\scriptsize Computed probability estimate to the child node}}
		\State $\calQ$.add(($v'$, , $\bx$))  \Comment{{\scriptsize Add the child node and its $P(v'|\bx)$ to the priority queue}}
	\EndFor	
\EndWhile
\end{algorithmic}
\label{alg:ubop-hf-next}
\end{small}
\end{algorithm}

\section{Related work}
\label{sec:relwork}

The paper that is the closest to our work is \citep{Delcoz2009LearningNC}, in which an efficient algorithm for the F$_\beta$-measure is proposed. Our work extends this paper in two directions: 1) we introduce a general optimization framework that generalizes the results of \citet{Delcoz2009LearningNC} to other utility functions, and 2) we also develop efficient algorithms that further improve their algorithm, which can be interpreted as a specific case of the SVBOP-Full algorithm.

We discussed several papers that introduce various set-based utility functions \citep{Corani2008NCC,Corani2009LNCC,Zaffalon2012EvaluatingCC,Yang2017b,Nguyen2018RelMCC,Ramaswamy2015CAMCRO}. Those papers mainly highlight the usefulness and properties of these functions, while focussing less on algorithmic aspects. From that perspective, we rather see our work as complementary instead of competing.     

In the literature, one can find several simple approaches to generate set-valued predictions. Arguably, the most simple approach is to look at the conditional class probabilities, and return a predefined number of $s$ classes; the classes with the highest conditional class probabilities (top-$s$ prediction). The main downside of this approach is that set-valued predictions for different instances will have the same cardinality. In practice this is often unwanted; small sets should be returned when the uncertainty about the true class label is small, while bigger sets are needed when the uncertainty becomes larger. Another simple approach is thresholding on conditional class probabilities. One can define a fixed threshold for $P(c \, \vert \, \bx)$ and return those classes that exceed this threshold, or one can define a  threshold on the cumulative probability of the resulting set. In the latter case, one first sorts the classes in decreasing order of conditional class probabilities. For a user-defined $\theta \in [0,1]$, one then returns the top-$s$ classes for which $s$ is given by
\begin{equation}
\label{eq:thresh}
\inf \Big\{s: \sum_{i=1}^{s} P(c^{(i)}\,|\,\bx) \geq \theta \Big\}\,,
\end{equation}
with $c^{(1)},\ldots,c^{(K)}$. Both thresholding approaches have a clear disadvantage: they only look at a specific threshold, and do not account for the fact that the size of the predicted set might considerably change if the threshold is slightly changed. Thresholding will be suboptimal in view of optimizing (\ref{eq:bayesoptimal}). This can be best observed from (\ref{eq:s}), which proves that the Bayes-optimal set not only depends on the sum of the first $s$ conditional probabilites, but also on the next probability in the sorted list, i.e., $p_{s+1}$. Nonetheless, top-$s$ prediction and thresholding are two obvious baselines that will be analyzed in our experiments. 

Set-valued predictions are also considered in hierarchical multi-class classification, mostly in the form of internal nodes of the hierarchy  \citep{Freitas_atutorial,Rangwala2017,Yang2017}. Compared to the ``flat'' multi-class case, the prediction space is thus restricted, because only sets of classes that correspond to nodes of the hierarchy can be returned as a prediction. In this paper, we do not consider such a setup, but the SVBOP-HF algorithm could be adjusted for that purpose. 

Set-based utility functions have been analyzed in the context of hierarchical multi-class classification. For example, \citet{Yang2017} evaluate various members of the $u_{\delta,\gamma}$ family in a framework where hierarchies are considered for computational reasons, while \citet{Oh2017TopKHC} optimizes recall by fixing $|\sety|$ as a user-defined parameter. Popular in hierarchical multi-label classification is the tree-distance loss, which could also be interpreted as a way of evaluating set-valued predictions (see e.g.\ \citep{Bi2015}). This loss is not a member of the family (\ref{eq:ufamily}), however. Besides, from the perspective of abstention in the case of uncertainty, it appears to be a less useful loss function, because it has the tendency to return large sets.  

Set-valued predictions are also produced in the framework of conformal prediction (CP) \citep{vovk_al,shaf_at08,bala_cp, Denis17ConfidenceSets}. This framework is rooted in statistical hypothesis testing, and in a sense can be seen as a ``frequentist statistics'' counterpart to our approach, which is more in the spirit of Bayesian decision theory. More specifically, given a new query instance $\bx$, CP constructs a prediction set or \emph{prediction region} $Y \subseteq \mathcal{Y}$ that is guaranteed to cover the true outcome $y$ with a pre-defined probability $1- \epsilon$ (for example 95\,\%). To this end, the hypothesis that $y = \hat{y}$ is tested for each candidate prediction $\hat{y} \in \mathcal{Y}$, and only those candidates for which the test can be rejected are excluded from $Y$. The test itself is non-parametric and leverages a ``nonconformity'' function $s: \, \mathcal{X} \times \mathcal{Y} \rightarrow \mathbb{R}$ that assigns scores $s(\vec{x}, y) \in \mathbb{R}$ to input/output tuples; the latter can be interpreted as a measure of ``strangeness'' of the pattern $(\bx, y)$, i.e., the higher the score, the less the data point $(\bx, y)$ conforms to what one would expect to observe. Roughly speaking, the hypothesis $y = \hat{y}$ is then rejected if the nonconformity score $s(\bx, \hat{y})$ is among the $\epsilon$\,\% highest of all scores $s(\bx_1, y_1) , \ldots , s(\vec{x}_N, y_N)$ observed in the (training) data so far. Conformal prediction has originally been introduced as an online learning method, but inductive variants have been developed as well \citep{papa_ic08}.

As we are maximizing a set-based utility function, our work is also somewhat connected to submodular optimization in machine learning \citep{Syed2016,Vondrak2019}. From (\ref{eq:topl}) it follows that $U(\sety,P,u)$ is either a submodular or supermodular function, when $g$ is concave or convex, respectively. In the first case, when $g$ is decreasing and concave, then $g(|\sety|)$ is submodular, and because $P(\sety \,|\, \bx)$ is modular, $U(\sety,P,u)$ must be submodular. One could therefore think of using submodular optimization algorithms to solve (\ref{eq:bayesoptimal}), but we believe that such algorithms would not be very useful. Theorem~\ref{thm:ksubsets} made us conclude that the combinatorial problem (\ref{eq:bayesoptimal}) could be reduced to a line search on $K+1$ sets. In contrast, a line search cannot be applied in combinatorial machine learning problems that are solved with submodular optimization techniques. To give one example, submodular optimization of (\ref{eq:bayesoptimal}) by means of a continuous relaxation via the Lovasz extension \citep{Vondrak2019} would still require the computation of $P(c \,|\,\bx)$ for all classes $c$, whereas the algorithms that we introduce avoid this. 

Currently, the development of efficient algorithms is an active theme of research in the area of extreme classification. The overwhelming majority of algorithms developed in this area focus on multi-label classification, but also some algorithms for multi-class problems, with a large number of classes, have been proposed. Such algorithms are not immediately applicable to the setting of set-valued prediction. Nevertheless, the idea of sorting probabilities is commonly used in extreme multi-label classification, for optimizing specific loss functions, such as the F$_{\beta}$-measure, precision, and normalized discounted cumulative gain, see e.g.\ \citep{Waegeman2014,Babbar2018}. For several of those measures, one typically either selects the top-$s$ scoring labels (with $s$ predefined), or those labels for which the marginal probability exceeds a threshold. Those two approaches are suboptimal in view of optimizing (\ref{eq:bayesoptimal}), but it is interesting to see how much we gain with our more complicated algorithms. That is something we will analyze in the experiments. 

Finally, set-valued prediction is also closely connected to uncertainty modelling for multi-class classification. For safety-critical applications such as self-driving cars and medical decision making, it is important to have an indication of uncertainty when making decisions. Some recently developed methods make a distinction between epistemic and aleatoric uncertainty, see e.g.\ \citep{hllermeier2019aleatoric} for an overview. The former type of uncertainty arises due to a lack of data for training, whereas the latter alludes to uncertainty that cannot be reduced by collecting more data, e.g.\ measurement noise, low-quality features, etc. In our framework, we only consider aleatoric uncertainty, because we produce a set based on the estimated conditional class probabilities $P(c\,|\, \bx)$. Approaches that analyze epistemic uncertainty take other properties of the data into account, such as generalizations of probability theory \citep{senge14,Nguyen2018RelMCC} or measures based on dropout resampling at test time \citep{KendallG17, DepewegHDU18}. Very recently, \citet{ziyin2019deep} combine the idea of set-based utility maximization with density estimation, producing an empty set in case of high epistemic uncertainty. This might be an interesting path for future work. 

\section{Experiments}
\label{sec:res:exps}

We perform three types of experiments to illustrate and benchmark the algorithms that we introduce. The datasets for all experiments are shown in Table~\ref{tab:empana:datasets}. In the following section, we first illustrate the practical relevance of set-valued prediction on both image and text classification datasets. In Section~\ref{sec:res:uti}, we evaluate the proposed exact algorithm, together with simple baseline methods that produce sets, for different set-based utility functions on a wide range of practical datasets. For the last group of experiments, in Section~\ref{sec:res:compalgo}, we compare the proposed exact and approximate algorithms by looking at runtime efficiency versus predictive performance. For a general discussion on the experimental setup, we refer the reader to the appendix.

\begin{table}[t]
\caption{\small Summary of image (top) and text (bottom) datasets used in all experiments. Notation: $K$ -- number of classes, $D$ -- number of features, $N$ -- number of samples }
\label{tab:empana:datasets}
\begin{center}
\begin{small}
\vskip -0in
\begin{tabular}{lrrrr}
\toprule
    \textbf{Dataset} & $\mathbf{K}$ & $\mathbf{D}$ & $\mathbf{N_{train}}$ & $\mathbf{N_{test}}$ \\
    \midrule
    \textbf{MNIST}~\citep{LeCunMNIST} & 10 & 32 & 33600 & 8400 \\
    \textbf{VOC 2006}~\citep{PascalVoc2006} & 10 & 25088 & 1398 & 1477 \\
    \textbf{VOC 2007}~\citep{PascalVoc2007} & 20 & 25088 & 2808 & 2841 \\
    \textbf{Caltech-101}~\citep{Caltech101} & 97 & 25088 & 4338 & 4339 \\
    \textbf{Caltech-256}~\citep{Caltech256} & 256 & 25088 & 14890 & 14890 \\
    \textbf{ALOI.BIN}~\citep{GeusebroekIJCV2005} & 1000 & 636911 & 90000 & 8000 \\
    \midrule
    \textbf{DBpedia}~\citep{DBpedia} & 219 & 483214 & 240942 & 96797 \\
    \textbf{Bacteria}~\citep{Bacteria} & 2631 & 2472 & 10576 & 2301 \\
    \textbf{Proteins}~\citep{LiWUXFLG18} & 3485 & 26276 & 11830 & 10179 \\
    \textbf{DMOZ}~\citep{PartalasKBAPGAA15} & 11939 & 833484 & 335068 & 38340 \\
    \textbf{LSHTC1}~\citep{PartalasKBAPGAA15} & 12166 & 381571 & 93805 & 34905 \\
    \bottomrule%
\end{tabular}%
\end{small}
\end{center}
\end{table}

\subsection{Illustrations on image datasets}
\label{sec:res:illsvp}
\begin{figure}[t]%
    \centering
    \subfloat[Predictions for three MNIST test samples, for increasing $\beta \in \{1,2,5\}$ and Shannon entropy $H$.]{\ubopmnisttable}\quad%
    \subfloat[VOC 2006 test sample with top-5 prediction {\it\{\textbf{sheep, \underline{cow}}, horse, car, motorbike\}}. SVBOP-Full candidates, for $\beta=1$, are indicated in bold.]{\includegraphics[width=0.3\columnwidth]{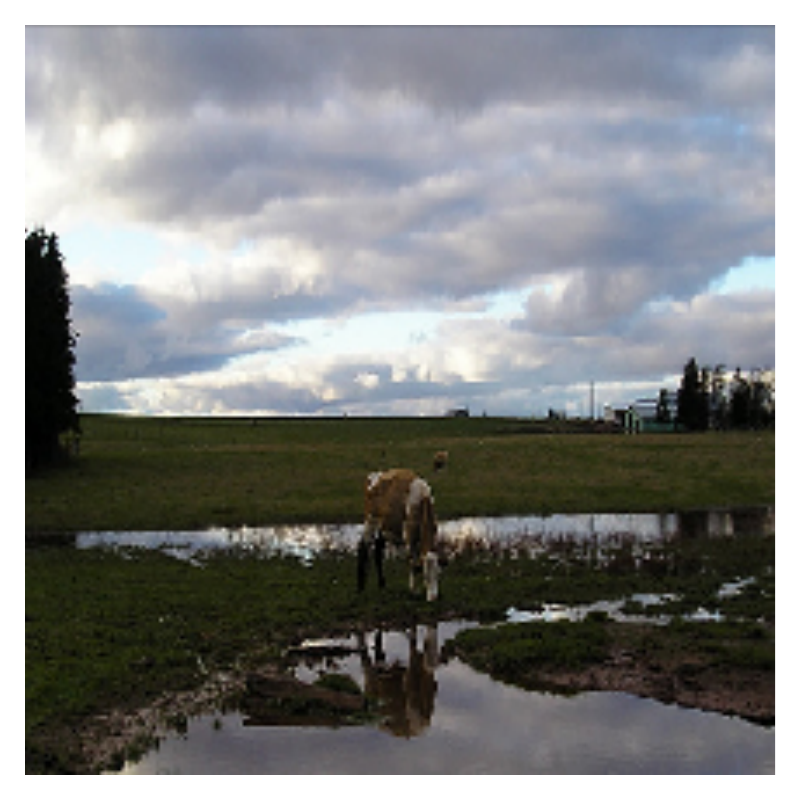}}\quad
    \subfloat[DBpedia test samples with corresponding top-5 predictions. SVBOP-Full candidates, for $\beta=1$, are indicated in bold.]{\ubopdbpediatable}\quad
    \caption{Set-valued predictions with SVBOP-Full and utility function $g_{F\beta}$ illustrated on MNIST, VOC 2006 and DBpedia. Ground truths are underlined in each prediction.}%
    \label{fig:ubopmnistvoc2006}%
\end{figure}
In the illustrative experiments we provide some examples that emphasize the practical usefulness of set-valued prediction. In Fig.~\ref{fig:ubopmnistvoc2006}(a) we show predictions obtained by the SVBOP-Full algorithm with utility function $g_{F\beta}$, on three MNIST test samples.  From left to right, we show three test instances for which the uncertainty (in the conditional class probabilities) is increasing. To this end, uncertainty is expressed by the Shannon entropy
$$H=\sum_{i=1}^K P(c_i\,|\,\bx)\log P(c_i\,|\,\bx)\,.$$ From top to down, we show predictions
for each image in function of decreasing $\beta\in\{5,2,1\}$. For increasing $\beta$ and uncertainty, the SVBOP-Full algorithm
typically predicts larger sets. For the last image, corresponding to number two, one can see
that predicting the class with the highest conditional class probability, i.e., the first element in the predicted
set\footnote{Note that the candidate solutions in the set are sorted in decreasing order of conditional class probability.}, would result in a false positive. However, for $\beta \in \{2,5\}$, the ground truth is returned as candidate solution in the set-valued predictions. 

We further illustrate the usefulness of set-valued prediction by looking at another image (VOC 2006) and a couple of text (DBpedia) examples in Fig.~\ref{fig:ubopmnistvoc2006}(b) and (c). There we show top-5 predictions and predictions obtained by SVBOP-Full (denoted in bold), by using utility function $g_{F1}$. For the VOC 2006 test image, with a cow and sheep in the background, the uncertainty reflected in the conditional class probabilities is high, most likely due to taking the picture in low light conditions. Again, returning the label with the highest conditional class probability (i.e., sheep) results in a false positive. A set-valued prediction by means of the simple top-5 or SVBOP-Full method, however, includes the ground truth as candidate solution for this particular case. A second observation is the suboptimality of the top-5 method, compared to SVBOP-Full. As the size of the prediction does not depend on the conditional class probabilities, but is rather fixed, we are at risk of including irrelevant candidates in the predicted set. These irrelevant candidates are most often characterized by low conditional class probabilities, for which an inclusion would result in a drop of expected utility in the SVBOP-Full Algorithm 
~(\ref{alg:ubop-framework}). For example, for the image of the cow in Fig.~\ref{fig:ubopmnistvoc2006}(b), two irrelevant candidates $\{car, motorbike\}$ are included in the top-5 prediction. The same can be observed for the examples in Fig.~\ref{fig:ubopmnistvoc2006}(c). 

\afterpage{\clearpage}
\begin{table}[t!]
\caption{\small Comparison of SVBOP-Full with baselines (thresholding, top-$s$ and inductive conformal prediction) in terms of optimizing different utility functions (listed in decreasing order of convexity: $u_{\delta=1.6,\gamma=0.6}, u_{F1}, u_{\delta=2.2,\gamma=1.2}$ and $u_{F5}$) for all datasets. Optimal utilities are underlined.}
\label{tab:comp:results}
\centering
\resizebox{0.85\columnwidth}{!}{%
\begin{tabular}{l|cccc|cccc}
\toprule
& \multicolumn{4}{c|}{\textbf{VOC 2006}} & \multicolumn{4}{c}{\textbf{VOC 2007}} \\
\textbf{Method} &  $\mathbf{u_{\delta=1.6,\gamma=0.6}}$ &  $\mathbf{u_{F1}}$ &  $\mathbf{u_{\delta=2.2,\gamma=1.2}}$ &  $\mathbf{u_{F5}}$ &  $\mathbf{u_{\delta=1.6,\gamma=0.6}}$ &  $\mathbf{u_{F1}}$ &  $\mathbf{u_{\delta=2.2,\gamma=1.2}}$ &  $\mathbf{u_{F5}}$ \\
\toprule
\textbf{SVBOP-Full-$\mathbf{u_{\delta=1.6,\gamma=0.6}}$} &                                 91.03 &              91.11 &                                 91.73 &              92.51 &                                 88.93 &              89.01 &                                 89.62 &              90.39 \\
\textbf{SVBOP-Full-$\mathbf{u_{F1}}$}                    &                                 91.13 &              91.22 &                                 91.90 &              92.82 &                                 88.84 &              88.94 &                                 89.64 &              90.62 \\
\textbf{SVBOP-Full-$\mathbf{u_{\delta=2.2,\gamma=1.2}}$} &                                 \underline{91.14} &              \underline{91.31} &                                 \underline{92.61} &              94.27 &                                 88.73 &              88.91 &                                 \underline{90.27} &              92.06 \\
\textbf{SVBOP-Full-$\mathbf{u_{F5}}$}                    &                                 87.76 &              88.35 &                                 90.74 &              \underline{97.12} &                                 83.60 &              84.29 &                                 86.80 &              \underline{94.74} \\
\midrule
\textbf{Threshold-$\mathbf{u_{\delta=1.6,\gamma=0.6}}$}  &                                 87.84 &              88.40 &                                 90.64 &              96.73 &                                 83.72 &              84.36 &                                 86.43 &              94.18 \\
\textbf{Threshold-$\mathbf{u_{F1}}$}                     &                                 87.84 &              88.40 &                                 90.64 &              96.73 &                                 83.72 &              84.36 &                                 86.43 &              94.18 \\
\textbf{Threshold-$\mathbf{u_{\delta=2.2,\gamma=1.2}}$}  &                                 87.84 &              88.40 &                                 90.64 &              96.73 &                                 83.72 &              84.36 &                                 86.43 &              94.18 \\
\textbf{Threshold-$\mathbf{u_{F5}}$}                     &                                 86.92 &              87.54 &                                 89.80 &              96.71 &                                 83.72 &              84.36 &                                 86.43 &              94.18 \\
\midrule
\textbf{Top-1}                                           &                                 90.72 &              90.72 &                                 90.72 &              90.72 &                                 \underline{88.95} &              88.95 &                                 88.95 &              88.95 \\
\textbf{Top-3}                                           &                                 45.91 &              49.19 &                                 59.03 &              91.35 &                                 44.94 &              48.15 &                                 57.78 &              89.43 \\
\textbf{Top-5}                                           &                                 29.46 &              33.18 &                                 39.01 &              86.26 &                                 29.03 &              32.69 &                                 38.44 &              84.99 \\
\midrule
\textbf{ICP$_{\mathbf{\epsilon=0.01}}$}                  &                                 87.06 &              87.64 &                                 90.09 &              96.23 &                                 70.44 &              71.67 &                                 74.65 &              92.45 \\
\textbf{ICP$_{\mathbf{\epsilon=0.10}}$}                  &                                 89.72 &              89.72 &                                 89.72 &              89.72 &                                 88.94 &              \underline{89.02} &                                 89.64 &              90.40 \\
\bottomrule
\end{tabular}
}\\
\vskip 0.02in
\resizebox{0.85\columnwidth}{!}{%
\begin{tabular}{l|cccc|cccc}
\toprule
& \multicolumn{4}{c|}{\textbf{Caltech-101}} & \multicolumn{4}{c}{\textbf{DBpedia}} \\
\textbf{Method} &  $\mathbf{u_{\delta=1.6,\gamma=0.6}}$ &  $\mathbf{u_{F1}}$ &  $\mathbf{u_{\delta=2.2,\gamma=1.2}}$ &  $\mathbf{u_{F5}}$ &  $\mathbf{u_{\delta=1.6,\gamma=0.6}}$ &  $\mathbf{u_{F1}}$ &  $\mathbf{u_{\delta=2.2,\gamma=1.2}}$ &  $\mathbf{u_{F5}}$ \\
\toprule
\textbf{SVBOP-Full-$\mathbf{u_{\delta=1.6,\gamma=0.6}}$} &                                 95.29 &              \underline{95.36} &                                 95.83 &              96.49 &                                 88.47 &              88.50 &                                 88.76 &              89.08 \\
\textbf{SVBOP-Full-$\mathbf{u_{F1}}$}                    &                                 95.24 &              95.33 &                                 \underline{95.84} &              96.67 &                                 88.49 &              88.52 &                                 88.82 &              89.18 \\
\textbf{SVBOP-Full-$\mathbf{u_{\delta=2.2,\gamma=1.2}}$} &                                 94.73 &              94.86 &                                 95.77 &              97.08 &                                 \underline{88.65} &              88.72 &                                 89.29 &              89.99 \\
\textbf{SVBOP-Full-$\mathbf{u_{F5}}$}                    &                                 87.89 &              88.41 &                                 90.10 &              96.93 &                                 88.57 &              \underline{88.76} &                                 90.07 &              91.91 \\
\midrule
\textbf{Threshold-$\mathbf{u_{\delta=1.6,\gamma=0.6}}$}  &                                 81.90 &              82.31 &                                 83.31 &              91.64 &                                 88.48 &              88.70 &                                 90.22 &              92.37 \\
\textbf{Threshold-$\mathbf{u_{F1}}$}                     &                                 81.90 &              82.31 &                                 83.31 &              91.64 &                                 88.48 &              88.70 &                                 90.22 &              92.37 \\
\textbf{Threshold-$\mathbf{u_{\delta=2.2,\gamma=1.2}}$}  &                                 81.90 &              82.31 &                                 83.31 &              91.64 &                                 88.09 &              88.38 &                                 \underline{90.30} &              93.16 \\
\textbf{Threshold-$\mathbf{u_{F5}}$}                     &                                 81.90 &              82.31 &                                 83.31 &              91.64 &                                 87.46 &              87.84 &                                 90.18 &              \underline{93.88} \\
\midrule
\textbf{Top-1}                                           &                                 95.07 &              95.07 &                                 95.07 &              95.07 &                                 88.22 &              88.22 &                                 88.22 &              88.22 \\
\textbf{Top-3}                                           &                                 46.02 &              49.31 &                                 59.17 &              91.57 &                                 45.48 &              48.73 &                                 58.47 &              90.49 \\
\textbf{Top-5}                                           &                                 29.37 &              33.08 &                                 38.90 &              86.01 &                                 29.22 &              32.91 &                                 38.70 &              85.56 \\
\midrule
\textbf{ICP$_{\mathbf{\epsilon=0.01}}$}                  &                                 89.31 &              89.83 &                                 91.64 &              \underline{97.69} &                                 65.04 &              66.39 &                                 72.30 &              89.84 \\
\textbf{ICP$_{\mathbf{\epsilon=0.10}}$}                  &                                 \underline{95.34} &              95.34 &                                 95.34 &              95.34 &                                 88.04 &              88.11 &                                 88.66 &              89.35 \\
\bottomrule
\end{tabular}
}\\
\vskip 0.02in
\resizebox{0.85\columnwidth}{!}{%
\begin{tabular}{l|cccc|cccc}
\toprule
& \multicolumn{4}{c|}{\textbf{Caltech-256}} & \multicolumn{4}{c}{\textbf{ALOI.BIN}} \\
\textbf{Method} &  $\mathbf{u_{\delta=1.6,\gamma=0.6}}$ &  $\mathbf{u_{F1}}$ &  $\mathbf{u_{\delta=2.2,\gamma=1.2}}$ &  $\mathbf{u_{F5}}$ &  $\mathbf{u_{\delta=1.6,\gamma=0.6}}$ &  $\mathbf{u_{F1}}$ &  $\mathbf{u_{\delta=2.2,\gamma=1.2}}$ &  $\mathbf{u_{F5}}$ \\
\toprule
\textbf{SVBOP-Full-$\mathbf{u_{\delta=1.6,\gamma=0.6}}$} &                                 81.90 &              82.04 &                                 83.06 &              84.46 &                                 96.38 &              96.41 &                                 96.66 &              96.97 \\
\textbf{SVBOP-Full-$\mathbf{u_{F1}}$}                    &                                 \underline{81.91} &              \underline{82.09} &                                 83.21 &              85.02 &                                 96.34 &              96.38 &                                 96.65 &              97.02 \\
\textbf{SVBOP-Full-$\mathbf{u_{\delta=2.2,\gamma=1.2}}$} &                                 80.97 &              81.27 &                                 \underline{83.39} &              86.39 &                                 96.25 &              96.33 &                                 \underline{96.86} &              97.56 \\
\textbf{SVBOP-Full-$\mathbf{u_{F5}}$}                    &                                 69.47 &              70.60 &                                 73.74 &              \underline{89.06} &                                 94.04 &              94.28 &                                 95.23 &              \underline{98.04} \\
\midrule
\textbf{Threshold-$\mathbf{u_{\delta=1.6,\gamma=0.6}}$}  &                                 69.73 &              70.48 &                                 72.28 &              85.59 &                                 89.56 &              89.80 &                                 90.54 &              94.14 \\
\textbf{Threshold-$\mathbf{u_{F1}}$}                     &                                 69.73 &              70.48 &                                 72.28 &              85.59 &                                 89.56 &              89.80 &                                 90.54 &              94.14 \\
\textbf{Threshold-$\mathbf{u_{\delta=2.2,\gamma=1.2}}$}  &                                 69.73 &              70.48 &                                 72.28 &              85.59 &                                 89.56 &              89.80 &                                 90.54 &              94.14 \\
\textbf{Threshold-$\mathbf{u_{F5}}$}                     &                                 69.73 &              70.48 &                                 72.28 &              85.59 &                                 89.56 &              89.80 &                                 90.54 &              94.14 \\
\midrule
\textbf{Top-1}                                           &                                 81.46 &              81.46 &                                 81.46 &              81.46 &                                 \underline{96.43} &              \underline{96.43} &                                 96.43 &              96.43 \\
\textbf{Top-3}                                           &                                 42.52 &              45.56 &                                 54.67 &              84.61 &                                 46.10 &              49.39 &                                 59.26 &              91.72 \\
\textbf{Top-5}                                           &                                 27.77 &              31.27 &                                 36.78 &              81.31 &                                 29.38 &              33.09 &                                 38.91 &              86.03 \\
\midrule
\textbf{ICP$_{\mathbf{\epsilon=0.01}}$}                  &                                 45.43 &              46.77 &                                 49.49 &              75.23 &                                 93.33 &              93.62 &                                 94.71 &              98.00 \\
\textbf{ICP$_{\mathbf{\epsilon=0.10}}$}                  &                                 77.10 &              77.77 &                                 80.95 &              87.80 &                                 96.18 &              96.18 &                                 96.18 &              96.18 \\
\bottomrule
\end{tabular}
}\\
\vskip 0.02in
\resizebox{0.85\columnwidth}{!}{%
\begin{tabular}{l|cccc|cccc}
\toprule
& \multicolumn{4}{c|}{\textbf{Bacteria}} & \multicolumn{4}{c}{\textbf{Proteins}} \\
\textbf{Method} &  $\mathbf{u_{\delta=1.6,\gamma=0.6}}$ &  $\mathbf{u_{F1}}$ &  $\mathbf{u_{\delta=2.2,\gamma=1.2}}$ &  $\mathbf{u_{F5}}$ &  $\mathbf{u_{\delta=1.6,\gamma=0.6}}$ &  $\mathbf{u_{F1}}$ &  $\mathbf{u_{\delta=2.2,\gamma=1.2}}$ &  $\mathbf{u_{F5}}$ \\
\toprule
\textbf{SVBOP-Full-$\mathbf{u_{\delta=1.6,\gamma=0.6}}$} &                                 92.91 &              93.00 &                                 93.61 &              94.45 &                                 70.52 &              70.56 &                                 70.91 &              71.33 \\
\textbf{SVBOP-Full-$\mathbf{u_{F1}}$}                    &                                 92.86 &              92.98 &                                 \underline{93.61} &              94.77 &                                 70.54 &              70.59 &                                 70.97 &              71.44 \\
\textbf{SVBOP-Full-$\mathbf{u_{\delta=2.2,\gamma=1.2}}$} &                                 91.37 &              91.57 &                                 92.85 &              94.81 &                                 70.75 &              70.85 &                                 71.69 &              72.71 \\
\textbf{SVBOP-Full-$\mathbf{u_{F5}}$}                    &                                 84.28 &              85.01 &                                 88.67 &              96.12 &                                 71.39 &              71.66 &                                 73.66 &              76.32 \\
\midrule
\textbf{Threshold-$\mathbf{u_{\delta=1.6,\gamma=0.6}}$}  &                                 88.80 &              89.28 &                                 91.53 &              96.41 &                                 71.41 &              71.76 &                                 74.21 &              77.58 \\
\textbf{Threshold-$\mathbf{u_{F1}}$}                     &                                 88.80 &              89.28 &                                 91.53 &              96.41 &                                 71.41 &              71.76 &                                 74.21 &              77.58 \\
\textbf{Threshold-$\mathbf{u_{\delta=2.2,\gamma=1.2}}$}  &                                 88.80 &              89.28 &                                 91.53 &              96.41 &                                 \underline{71.42} &              \underline{71.90} &                                 \underline{75.19} &              79.93 \\
\textbf{Threshold-$\mathbf{u_{F5}}$}                     &                                 88.80 &              89.28 &                                 91.53 &              \underline{96.41} &                                 71.42 &              71.90 &                                 75.19 &              79.93 \\
\midrule
\textbf{Top-1}                                           &                                 93.16 &              93.16 &                                 93.16 &              93.16 &                                 70.24 &              70.24 &                                 70.24 &              70.24 \\
\textbf{Top-3}                                           &                                 45.83 &              49.11 &                                 58.93 &              91.20 &                                 42.31 &              45.33 &                                 54.40 &              84.19 \\
\textbf{Top-5}                                           &                                 29.33 &              33.03 &                                 38.84 &              85.87 &                                 27.69 &              31.18 &                                 36.67 &              81.07 \\
\midrule
\textbf{ICP$_{\mathbf{\epsilon=0.01}}$}                  &                                 50.68 &              53.05 &                                 63.38 &              91.09 &                                 34.78 &              36.19 &                                 43.03 &              60.80 \\
\textbf{ICP$_{\mathbf{\epsilon=0.10}}$}                  &                                 \underline{93.46} &              \underline{93.46} &                                 93.46 &              93.46 &                                 64.12 &              65.48 &                                 72.32 &              \underline{86.28} \\
\bottomrule
\end{tabular}
}
\\
\vskip 0.02in
\resizebox{0.85\columnwidth}{!}{%
\begin{tabular}{l|cccc|cccc}
\toprule
& \multicolumn{4}{c|}{\textbf{DMOZ}} & \multicolumn{4}{c}{\textbf{LSHTC1}} \\
\textbf{Method} &  $\mathbf{u_{\delta=1.6,\gamma=0.6}}$ &  $\mathbf{u_{F1}}$ &  $\mathbf{u_{\delta=2.2,\gamma=1.2}}$ &  $\mathbf{u_{F5}}$ &  $\mathbf{u_{\delta=1.6,\gamma=0.6}}$ &  $\mathbf{u_{F1}}$ &  $\mathbf{u_{\delta=2.2,\gamma=1.2}}$ &  $\mathbf{u_{F5}}$ \\
\midrule
\textbf{SVBOP-Full-$\mathbf{u_{\delta=1.6,\gamma=0.6}}$} &                                 \underline{41.14} &              \underline{41.48} &                                 \underline{43.35} &              46.85 &                                 42.61 &              42.72 &                                 43.52 &              44.60 \\
\textbf{SVBOP-Full-$\mathbf{u_{F1}}$}                    &                                 40.43 &              40.91 &                                 42.83 &              48.28 &                                 42.63 &              42.78 &                                 43.66 &              45.06 \\
\textbf{SVBOP-Full-$\mathbf{u_{\delta=2.2,\gamma=1.2}}$} &                                 39.51 &              40.14 &                                 43.31 &              49.80 &                                 \underline{42.65} &              \underline{42.88} &                                 \underline{44.49} &              46.75 \\
\textbf{SVBOP-Full-$\mathbf{u_{F5}}$}                    &                                 19.14 &              19.97 &                                 21.42 &              39.76 &                                 39.14 &              40.01 &                                 42.46 &              54.02 \\
\midrule
\textbf{Threshold-$\mathbf{u_{\delta=1.6,\gamma=0.6}}$}  &                                 10.73 &              10.88 &                                 11.06 &              17.54 &                                 35.92 &              36.18 &                                 37.01 &              41.36 \\
\textbf{Threshold-$\mathbf{u_{F1}}$}                     &                                 10.73 &              10.88 &                                 11.06 &              17.54 &                                 35.92 &              36.18 &                                 37.01 &              41.36 \\
\textbf{Threshold-$\mathbf{u_{\delta=2.2,\gamma=1.2}}$}  &                                 10.73 &              10.88 &                                 11.06 &              17.54 &                                 35.92 &              36.18 &                                 37.01 &              41.36 \\
\textbf{Threshold-$\mathbf{u_{F5}}$}                     &                                 10.73 &              10.88 &                                 11.06 &              17.54 &                                 35.92 &              36.18 &                                 37.01 &              41.36 \\
\midrule
\textbf{Top-1}                                           &                                 40.41 &              40.41 &                                 40.41 &              40.41 &                                 42.00 &              42.00 &                                 42.00 &              42.00 \\
\textbf{Top-3}                                           &                                 25.99 &              27.84 &                                 33.41 &              51.71 &                                 27.13 &              29.06 &                                 34.88 &              53.98 \\
\textbf{Top-5}                                           &                                 18.27 &              20.57 &                                 24.19 &              \underline{53.49} &                                 18.96 &              21.35 &                                 25.10 &              \underline{55.50} \\
\midrule
\textbf{ICP$_{\mathbf{\epsilon=0.01}}$}                  &                                  1.39 &               1.45 &                                  1.57 &               2.99 &                                  2.46 &               2.61 &                                  2.84 &               6.32 \\
\textbf{ICP$_{\mathbf{\epsilon=0.10}}$}                  &                                  2.55 &               2.82 &                                  3.10 &              13.30 &                                 14.46 &              15.09 &                                 16.30 &              29.59 \\
\bottomrule
\end{tabular}
}
\end{table}

\subsection{Comparison of different utility functions and baselines}
\label{sec:res:uti}

The goal of the second type of experiments is to evaluate the SVBOP-Full algorithm for several utility functions. Two different parameterizations of the $u_{\delta,\gamma}$ and $u_{F\beta}$ families are studied, leading to four utility functions in total. We benchmark the SVBOP-Full algorithm against several baselines (which are all described in the related work section): 
\begin{enumerate}
\item Thresholding: predicting those classes for which the total cumulative probability mass exceeds a user-defined threshold, as explained by (\ref{eq:thresh}). For each of the four utility functions we tune the threshold on a validation set, by considering ten equally-spaced values. As a result, we obtain four different thresholding strategies, each of them tailored for a specific utility function.
\item Top-$s$: predicting a set consisting of the $s$ classes with highest conditional class probabilities. Here we consider three versions, with $s\in\{1, 3, 5\}$. As discussed in the related work section, top-$s$ returns a fixed number of classes for each instance, which can be considered as a limitation. However, it is interesting to see how this suboptimal approach performs w.r.t.\ set-based utility functions.  
\item Inductive conformal prediction (ICP): we experiment with the commonly-used nonconformity function $s(\bx, y) = 1-P(y\,|\,\bx)\,,$
and consider two fixed significance levels $\epsilon\in\{0.01, 0.10\}$. 
\end{enumerate} 

Table~\ref{tab:comp:results} shows for all methods the results obtained on test sets, where the highest obtained utilities are underlined. The utility functions are ordered in decreasing order of convexity: $u_{\delta=1.6,\gamma=0.6}, u_{F1}, u_{\delta=2.2,\gamma=1.2}$ and $u_{F5}$. The first three utility functions all behave very similar to precision, which explains why the results are similar. Due to a higher convexity, these utility functions give a high reward to small sets, such that the top-1 in general yields very good results for those utility functions. At the other side, for $u_{F5}$ the picture looks different; there top-3 or top-5 are often much better than top-1, because this utility function gives a higher reward to larger sets. 

The performance of the SVBOP-Full algorithm is in accordance with our theoretical results. In general, it is one of the best-performing methods for all datasets and utility functions that were analyzed. However, the differences with the other methods are small. This is of course not very surprising, because all tested inference algorithms depart from the same conditional class probabilities. Differences in performance can only be attributed to (relatively small) differences in the inference algorithms.  As discussed in Section~\ref{sec:relwork}, thresholding is not Bayes-optimal w.r.t.\ (\ref{eq:bayesoptimal}), but on the analyzed datasets it performs quite well. We can conclude that the theoretical shortcomings of thresholding will only lead to small performance drops in practice. 

Finally, inductive conformal prediction performs quite well on some datasets, but this method yields bad results on other datasets. This can of course be explained by the fact that inductive conformal prediction does not intend to maximize a utility function. As explained in Section~\ref{sec:relwork}, this method rather intends to return sets that contain the true class label with high confidence. This phenomenon is especially visible on the LSHTC and DMOZ datasets, where $K$ is large. Then, inductive conformal prediction will produces very large sets, when it wants to cover the true class with high probability. 

\subsection{Comparison of exact and approximate algorithms on large datasets}
\label{sec:res:compalgo}

In the final group of experiments, we would like to compare the proposed exact and approximate algorithms by looking at runtime efficiency versus predictive performance. Table~\ref{tab:empana:results} summarizes the results for the SVBOP-Full, SVBOP-HSG, and SVBOP-HF approaches, obtained on the five largest datasets (w.r.t. the number of classes). We use the same weights for SVBOP-Full and SVBOP-HSG algorithms. For SVBOP-HF, we consider a predefined hierarchy, if available, and a hierarchy constructed during training by means of hierarhical balanced 2-means on class profiles, as explained in Section~\ref{sec:algorithmic_solutions}. For each algorithm we optimize two different utility functions: $u_{F1}$ and $u_{\delta=2.2, \gamma=1.2}$. We report train and test time, as well as average utility, recall ($\mathbbm{1}_{y_{i}\in\sety(\bx_{i})}$) and size of the predicted sets. Additionally, we also include average recall for top-1 predictions; this is in essence accuracy.  

For almost all datasets, SVBOP-Full yields the best predictive performance while being, as expected, always the slowest. 
For all datasets, SVBOP-HSG  achieves a predictive performance that is very close to SVBOP-Full, while being at the same time even a few times faster in inference on DMOZ and LSTHC1 datasets. 
Unsurprisingly, hierarchical factorization leads to the highest speedup both in training and inference. However, for most datasets, it comes at the expense of predictive performance. Only for datasets where a meaningful natural hierarchy is given (i.e. biological datasets), SVBOP-HF$_p$ outperforms SVBOP-Full and SVBOP-HF$_c$. 

In general, we observe that for almost all datasets, both approximate algorithms behave similarly to SVBOP-Full and manage to significantly improve recall with an only small increase in average prediction size. At the same time, the approximate algorithms improve the test times at the cost of predictive performance. This is not very surprising, as one might expect a clear trade-off between the two. In practice, the choice of a particular method should depend on the desired trade-off between runtime and predictive performance. 

\begin{table}[t]
    \caption{\small Performance versus runtime for the SVBOP-Full, SVBOP-HSG and SVBOP-HF algorithms, tested on five benchmark datasets for F1-measure utility ($u_{F1}$) and credal utility with $\delta = 2.2$ and $\gamma = 1.2$ ($u_{\delta,\gamma}$). Notation: R -- recall, $u$ -- utility value, $|\sety|$ -- prediction size, $t_{train}$ -- CPU train time in seconds, $t_{test}$ -- CPU test time in milliseconds / number of test samples, $p$ -- predefined hierarchy,  $c$ -- hierarchy built with hierarchical balanced 2-means clustering}
    \label{tab:empana:results}
    \resizebox{\textwidth}{!}{
    \begingroup
    \hspace{-8pt}
    \setlength{\tabcolsep}{3pt}
    \begin{tabular}{l|l|r|c|cccr|cccr}
        \toprule
        \textbf{Dataset} & \textbf{Algo.} & $\mathbf{t_{train}}$ & \textbf{Top-1} &
        $\mathbf{u_{F1}}$ & \textbf{R} & $|\sety|$ & $\mathbf{t_{test}}$ &
        $\mathbf{u_{\delta,\gamma}}$ & \textbf{R} & $|\sety|$ & $\mathbf{t_{test}}$ \\
        \midrule

        \textbf{ALOI.BIN} & Full & 5065 & 96.43 & 96.38 & 97.11 & 1.03 & 4.89 & 96.86 & 97.68 & 1.05 & 4.74 \\
        & HSG & 5087 & 96.41 & 96.23 & 96.96 & 1.04 & 3.09 & 96.66 & 97.54 & 1.05 & 3.11 \\
        & HF$_c$ & 163 & 93.15 & 93.46 & 95.16 & 1.06 & 0.29 & 93.97 & 95.44 & 1.10 & 0.27  \\
        \midrule

        \textbf{Bacteria} & Full & 6303 & 93.16 & 92.98 & 94.55 & 1.14 & 5.02 & 92.85 & 94.64 & 1.18 & 4.61 \\
        & HSG & 6323 & 92.45 & 93.16 & 94.11 & 1.09 & 1.83 & 93.58 & 94.58 & 1.12 & 1.92 \\
        & HF$_p$ & 360 & 91.19 & 91.38 & 92.97 & 1.06 & 0.18 & 91.42 & 92.98 & 1.09 & 0.20 \\
        & HF$_c$ & 70 & 90.72 & 91.14 & 92.58 & 1.06 & 0.08 & 91.08 & 92.59 & 1.09 & 0.09 \\
        \midrule

        \textbf{Proteins} & Full & 2192 & 70.24 & 70.59 & 70.98 & 1.31 & 15.53 & 71.69 & 71.73 & 1.36 & 14.28 \\
        & HSG & 2672 & 69.58 & 69.83 & 70.61 & 1.22 & 10.95 & 69.46 & 70.38 & 1.27 & 11.35 \\
        & HF$_p$ & 77 & 81.59 & 81.80 & 82.54 & 1.03 & 0.48 & 82.21 & 82.57 & 1.05 & 0.43 \\
        & HF$_c$ & 22 & 79.73 & 79.95 & 80.92 & 1.04 & 0.17 & 80.10 & 80.67 & 1.07 & 0.17 \\
        \midrule

        \textbf{DMOZ} & Full & 82872 & 40.41 & 40.91 & 51.33 & 3.52 & 55.04 & 43.31 & 52.46 & 2.73 & 54.47 \\
        & HSG & 83181 & 39.97 & 40.13 & 49.66 & 3.26 & 2.67 & 42.02 & 50.05 & 2.36 & 2.72 \\
        & HF$_c$ & 722 & 38.03 & 22.79 & 46.70 & 7.15 & 11.94 & 25.79 & 45.44 & 5.32 & 11.01 \\
        \midrule

        \textbf{LSHTC1} & Full & 71509 & 42.00 & 42.78 & 45.38 & 1.29 & 46.13 & 44.49 & 47.24 & 1.44 & 48.37 \\
        & HSG & 72361 & 41.52 & 42.30 & 44.86 & 1.30 & 8.28 & 43.99 & 46.71 & 1.44 & 9.71 \\
        & HF$_p$ & 557 & 39.82 & 40.96 & 44.79 & 1.42 & 0.52 & 43.20 & 47.21 & 1.60 & 0.60 \\
        & HF$_c$ & 338 & 38.53 & 39.19 & 43.36 & 1.48 & 1.15 & 41.38 & 45.78 & 1.66 & 1.24 \\
        \bottomrule%
    \end{tabular}%
    \endgroup%
    }%
\end{table}

\section{Conclusion}
We introduced a decision-theoretic framework for a general family of set-based utility functions, including most of the measures used in the literature so far, and developed three Bayes-optimal inference algorithms that exploit specific assumptions to improve runtime efficiency. 
Depending on the concrete dataset, those assumptions may or may not affect  predictive performance.

In future work, we plan to extend our decision-theoretic framework toward uncertainty representations more general than standard probability, for example taking up a distinction between so-called aleatoric and epistemic uncertainty recently put forward by several authors \citep{senge14,KendallG17, DepewegHDU18, Nguyen2018RelMCC}.

\bibliographystyle{spbasic}     
\bibliography{main}

\newpage
\appendix

\section{Regret bounds for the utility functions}

In this part we present a short theoretical analysis that relates the Bayes optimal solution for the set-based utility functions 
to the solution obtained on the probabilities given by a trained model. 
The goal is to upper bound the regret of the set-based utility functions by the $L_1$ error of the class probability estimates.
The analysis is performed on the level of a single $\bx$. 

Let $\hat P(\bx)$ be the estimate of the true underlying distribution $P(\bx)$. 
Let $U^*(P, u)$ denote the optimal utility for $P$ obtained by the optimal solution $\sety^*$ (this solution does not have to be unique). 
Now, let $\sety$ denote the optimal solution with respect to $\hat P(\bx)$. 
We define the regret of $\sety$ as:
\begin{eqnarray*}
\mathrm{reg}_u(\sety) & = & U^*(P, u) - U(\sety, P, u) \\ 
& = & \sum_{c \in \mathcal{Y}} u(c,\sety^*) P(c\,|\,\bx) - \sum_{c \in \mathcal{Y}} u(c,\sety) P(c\,|\,\bx) \\
& = &  \sum_{c \in \mathcal{Y}} \left ( u(c,\sety^*) - u(c,\sety) \right ) P(c\,|\,\bx)
\end{eqnarray*}

We bound $\mathrm{reg}_u(\hat P(\bx))$ in terms of the $L_1$-estimation error, i.e.:
$$
\sum_{c \in \mathcal{Y}} | P(c\,|\,\bx) - \hat P(c\,|\,\bx) |
$$

Note that if $\sety^* = \sety$ the regret is 0. Otherwise, we need to have 
$$
U(\sety, \hat P, u) \ge U(\sety^*, \hat P, u)
$$ 
Thus, we can write 
\begin{eqnarray}
\mathrm{reg}_u(\sety) & \le & U^*(P, u) - U(\sety, P, u)  + U(\sety, \hat P, u) - U(\sety^*, \hat P, u) \nonumber \\
& = & \sum_{c \in \mathcal{Y}} \left ( u(c,\sety^*) - u(c,\sety) \right ) P(c\,|\,\bx)
+ \sum_{c \in \mathcal{Y}} \left ( u(c,\sety) - u(c,\sety^*) \right ) \hat P(c\,|\,\bx) \nonumber \\
& = & \sum_{c \in \mathcal{Y}}  u(c,\sety^*) \left ( P(c\,|\,\bx) - \hat P(c\,|\,\bx) \right )
+ \sum_{c \in \mathcal{Y}} u(c,\sety) \left (\hat P(c\,|\,\bx)  - P(c\,|\,\bx) \right )  \nonumber \\
& \le & \sum_{c \in \mathcal{Y}}  u(c,\sety^*) \left | P(c\,|\,\bx) - \hat P(c\,|\,\bx) \right |
+ \sum_{c \in \mathcal{Y}} u(c,\sety) \left | \hat P(c\,|\,\bx)  - P(c\,|\,\bx) \right |  \label{eqn:absolute_value} \\
& = & \sum_{c \in \mathcal{Y}} \left ( u(c,\sety^*) + u(c,\sety)  \right )\left | P(c\,|\,\bx) - \hat P(c\,|\,\bx) \right | \nonumber \\
& \le & 2 \sum_{c \in \mathcal{Y}} \left | P(c\,|\,\bx) - \hat P(c\,|\,\bx) \right | \label{eqn:bounded_utility} 
\end{eqnarray}
The inequality in (\ref{eqn:absolute_value}) follows from the properties of the absolute function, $a \le |a|$, 
while the one in (\ref{eqn:bounded_utility})  holds because the utility functions are from the bounded interval, $u(\cdot,\cdot) \in [0,1]$.
We clearly see that the regret is upper bounded by the quality of the estimated probability distribution.

\section{Generalized reject option utility and parameter bounds}
\label{app:parambound}

In this part we analyze which values $\alpha$ and $\beta$ can take so that the $g_{\alpha,\beta}$ family is lower bounded by precision. This family is visualized in Figure~\ref{fig:gab}. For a given $K$, the following inequality must hold $\forall s\in \{1,\ldots,K\}$, such that $g_{\alpha,\beta}(s)$ is lower bounded by precision:
$$g_{\alpha, \beta}(s) \geq g_P(s)\,,$$
with utilities:
$$ g_{\alpha, \beta}(s) = 1-\alpha\Big(\frac{s-1}{K-1}\Big)^\beta\,,\quad g_P(s) = \frac{1}{s}\,. $$
When looking at the boundary cases (i.e., $s=1, s=K$), we find that:
$$\alpha \leq \frac{K-1}{K}\,.$$
By fixing $\alpha=\frac{K-1}{K}$, the above inequality can be rewritten, $\forall s\in \{2,\ldots,K-1\}$, as:
\begin{eqnarray*}
&& 1-\Big(\frac{K-1}{K}\Big)\Big(\frac{s-1}{K-1}\Big)^{\beta} \geq \frac{1}{s} \\
&\Leftrightarrow& \Big(\frac{s-1}{K-1}\Big)^{\beta} \leq \frac{K}{s}\Big(\frac{s-1}{K-1}\Big) \\
&\Leftrightarrow& \beta \geq \log_{\frac{s-1}{K-1}}\frac{K}{s}+1 \\
&\Rightarrow&  \beta \geq \log_{\frac{1}{K-1}}\frac{K}{2}+1\\
\end{eqnarray*}
Note that in the limit, when $K\rightarrow\infty$, we obtain the following upper and lower bound for $\alpha$ and $\beta$, respectively:
$$ \lim_{K\rightarrow\infty} \frac{K-1}{K} = 1\,,\quad \lim_{K\rightarrow\infty} \log_{\frac{1}{K-1}}\frac{K}{2}+1 = 0\,. $$

\begin{figure}[ht]
\centering
\begin{tabular}{cc}
\hskip -0.2in
\includegraphics[scale=0.35]{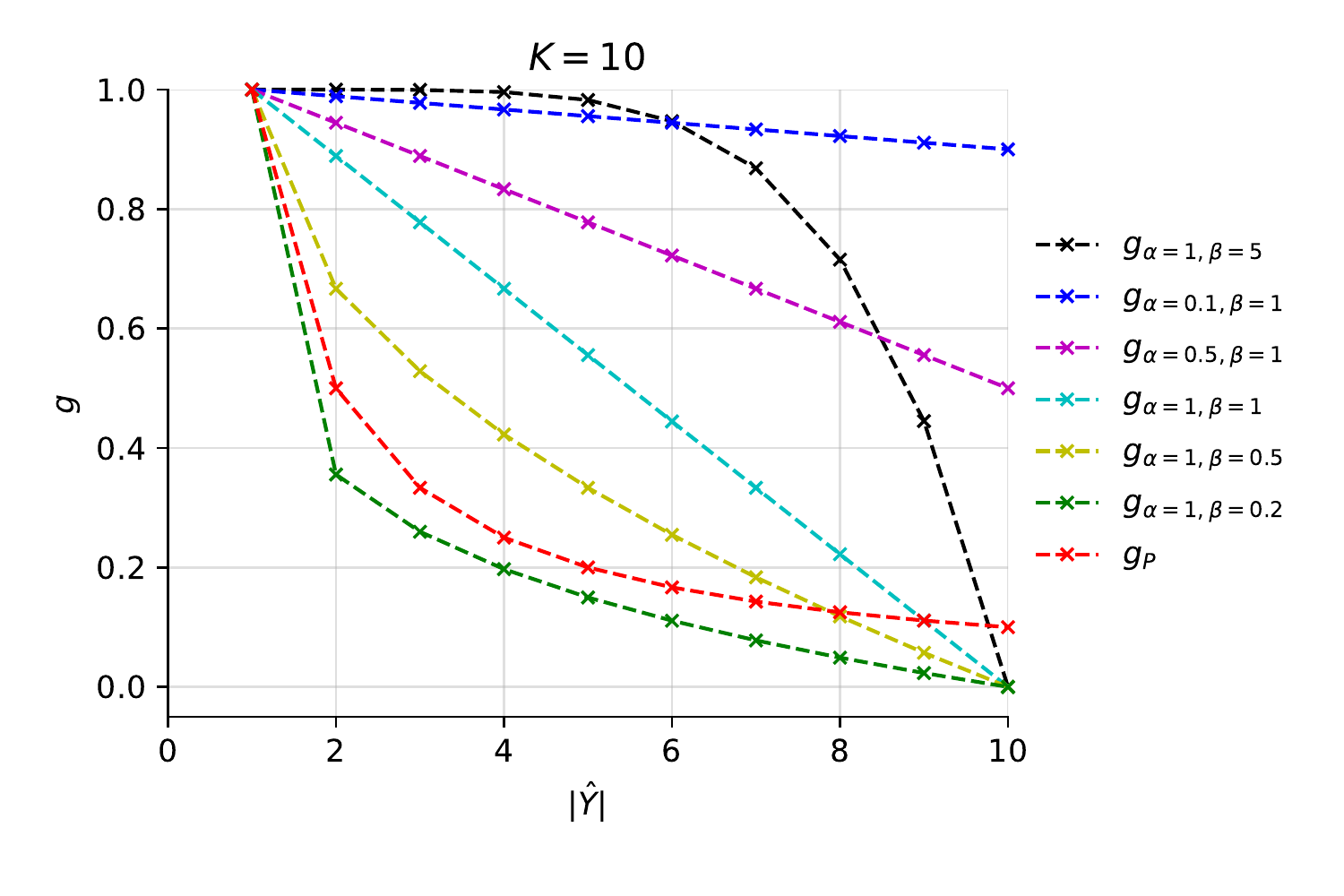} &
\includegraphics[scale=0.35]{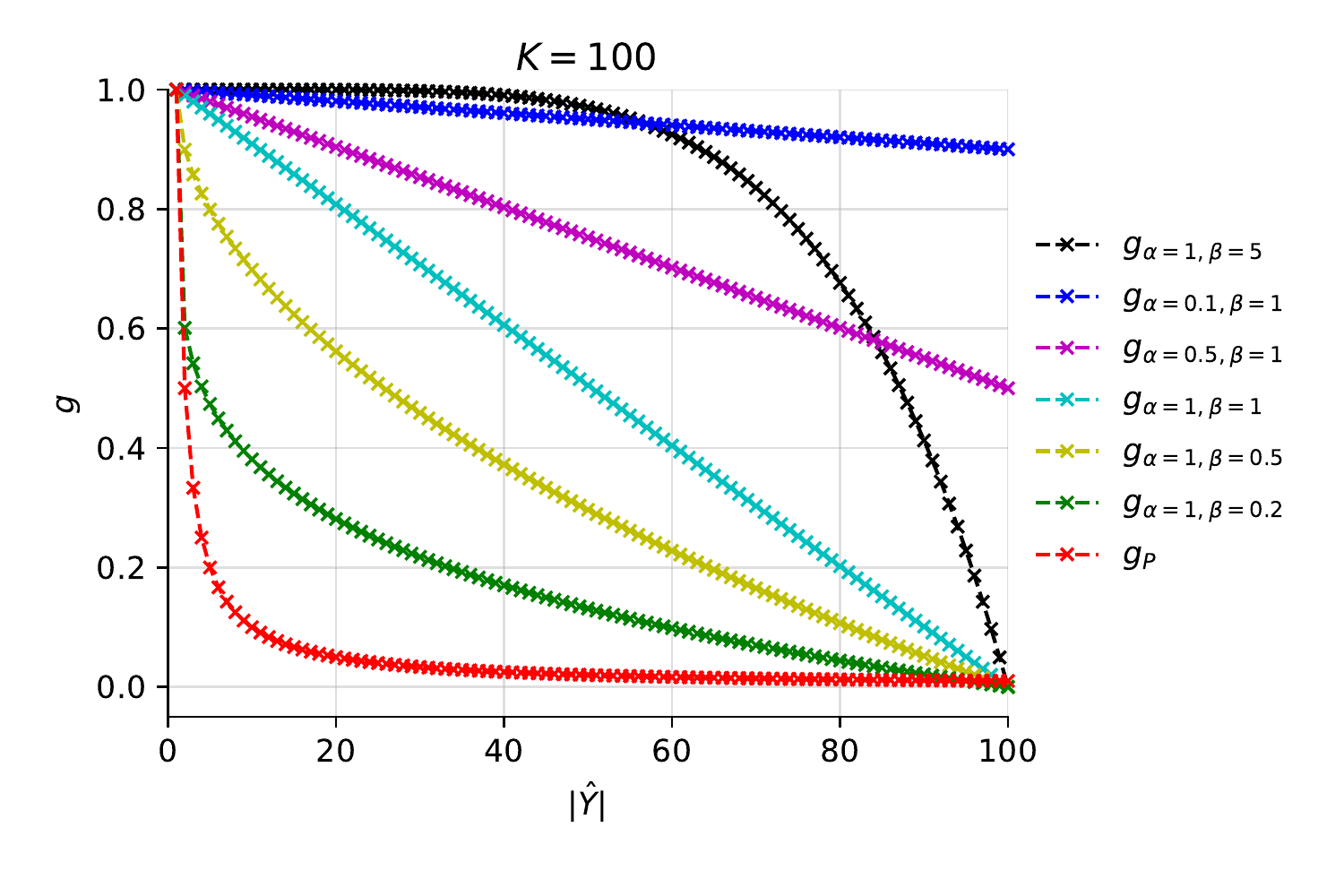}
\end{tabular}
\caption{A visualization of $g_{\alpha,\beta}$ in function of different values of $|\sety|$ and $K$.}
\label{fig:gab}
\end{figure}

\section{Experimental setup}
For all image datasets, except ALOI.BIN, we use hidden representations obtained by convolutional neural networks, whereas for the text datasets (bottom) tf-idf representations are used. The dimensionality of the representations are denoted by $D$. For the MNIST dataset we use a convolutional neural network with three consecutive convolutional, batchnorm and max-pool layers, followed by a fully connected dense layer with 32 hidden units. We use ReLU activation functions and optimize the categorical cross-entropy loss by means of Adam optimization with learning rate $\eta = 1e-3$. For the VOC 2006\footnote{The multi-label VOC datasets are transformed to multi-class by removing instances with more than one label.\label{fn:voc}}, VOC 2007\footref{fn:voc}, Caltech-101 and Caltech-256, the hidden representations are obtained by resizing images to 224x224 pixels and passing them through the convolutional part of an entire VGG16 architecture, including a max-pooling operation~\citep{Simonyan2014DCNN}. The weights are set to those obtained by training the network on ImageNet. For all convolutional neural networks, the number of epochs are set to 100 and early stopping is applied with a patience of five iterations. For ALOI.BIN, we use the ALOI dataset with random binning features, obtained by using the Laplacian kernel, as described in \citep{RahimiNIPS08}. Training is performed end-to-end on a GPU, by using the PyTorch library~\citep{paszke2017automatic} and infrastructure with the following specifications:
\begin{itemize}
    \item \textbf{CPU:} i7-6800K 3.4 GHz (3.8 GHz Turbo Boost) – 6 cores / 12 threads.
    \item \textbf{GPU:} 2x Nvidia GTX 1080 Ti 11GB + 1x Nvidia Tesla K40c 11GB.
    \item \textbf{RAM:} 64GB DDR4-2666.
\end{itemize}

For the bacteria dataset, tf-idf representations are calculated by using 3-, 4-, and 5-grams extracted from each 16S rRNA sequence in the dataset~\citep{FiannacaPRBRRGU18}. For the proteins dataset, we consider 3-grams in order to calculate the tf-idf representation for each protein sequence. To comply with literature, we concatenate the tf-idf representations with functional domain encoding vectors, which provide distinct functional and evolutional information about the protein sequence. For more information about the functional domain encodings, we refer the reader to~ \citep{LiWUXFLG18}. 

Finally, we use the learned hidden representations for the image datasets and calculate tf-idf representations for the text datasets to train the probabilistic models using a dual L2-regularized logistic regression model. For the DMOZ and LSHTC1 dataset we enforce sparsity by clipping all the learned weights less than a threshold $\eta = 0.1$ to zero~\citep{Babbar_at_el_2017}.
We implemented all SVBOP algorithms in C++ using Liblinear library~\citep{REF08a} and H-NSW implementation from NMSLIB~\citep{Naidan_and_Boytsov_2015}. All experiments were conducted on Intel Xeon E5-2697 v3 2.60GHz (14 cores) with 64GB RAM. We include detail information about selection of hyperparameters for all the models in the next section.

\section{Hyperparameters}
\begin{table}[t]
\caption{Values of hyperparameters used for SVBOP-Full, SVBOP-HSG, and SVBOP-HF algorithms for different datasets.}
\label{tab:hyperparams}
\begin{center}
\begingroup
\setlength{\tabcolsep}{3pt}
\begin{tabular}{l|rr|rrr|rrrr}
\toprule
 & \multicolumn{2}{c|}{Full} & \multicolumn{3}{c|}{HSG} & \multicolumn{4}{c}{HF} \\
\textbf{Dataset} & $C$ & $\epsilon_{l}$ & $M$ & $\textit{ef}_c$ & $k$ & $C$ & $\epsilon_{l}$ & $l$ & $\epsilon_{c}$ \\
\midrule 
\textbf{VOC 2006} & 100 & 0.1 & - & - & - & - & - & - & - \\
\textbf{VOC 2007} & 100 & 0.1 & - & - & - & - & - & - & - \\
\textbf{Caltech-101} & 100 & 0.1 & - & - & - & - & - & - & - \\
\textbf{Caltech-256} & 100 & 0.1 & - & - & - & - & - & - & - \\
\textbf{DBpedia} & $10^5$ & 0.1 & - & - & - & - & - & - & - \\
\textbf{ALOI.BIN} & 100 & 0.1 & 10 & 50 & 10 & 500 & 0.1 & 20 & 0.001 \\
\textbf{Bacteria} & $10^6$ & 0.1 & 50 & 200 & 100 & $10^6$ & 0.1 & 20 & 0.001 \\
\textbf{Proteins} & $10^6$ & 0.1 & 50 & 200 & 200 & $10^9$ & 0.1 & 20 & 0.001 \\
\textbf{Dmoz} & 1000 & 0.1 & 20 & 100 & 100 & 50 & 0.1 & 100 & 0.001 \\
\textbf{LSHTC1} & 1000 & 0.1 & 20 & 100 & 100 & 50 & 0.1 & 100 & 0.001 \\
\bottomrule%
\end{tabular}%
\endgroup
\end{center}
\end{table}
For Liblinear library, used for implementations of all SVBOP algorithms, we tuned two parameters: $C$~--~inverse of the regularization strength and $\epsilon_{l}$~--~ tolerance of termination criterion. For SVBOP-HSG and underlying H-NSW index method, we tuned four parameters: $M$~--~the maximum number of neighbors in the layers of H-NSW index, $\textit{ef}_c$~--~size of the dynamic candidate list during H-NSW index construction, $k$~--~initial size of the query to H-NSW index, $\textit{ef}_s$~--~size of the dynamic candidate list during H-NSW index query, was always set to the current value of $k$. For balanced 2-means tree building, we tuned two parameters: $l$~--~maximum number of leaves on the last level of a tree and $\epsilon_{c}$~--~§tolerance of termination criterion of the 2-means algorithm. We list all the hyperparameters we used to obtained all the results presented in Section~\ref{sec:res:uti} and Section~\ref{sec:res:compalgo} in Table~\ref{tab:hyperparams}

\end{document}